\documentclass{article}
\usepackage{iclr2027_conference,times}

\usepackage{amsmath,amssymb,amsthm,mathtools,mathrsfs}
\usepackage{thm-restate}
\usepackage{microtype}

\usepackage{algorithm}
\usepackage{algpseudocode}
\usepackage{booktabs}
\usepackage{array}
\usepackage{arydshln}
\usepackage{wrapfig}
\usepackage{xcolor}

\usepackage[inline]{enumitem}
\usepackage[most]{tcolorbox}
\usepackage{tikz}
\usetikzlibrary{arrows.meta,positioning,calc,fit}

\usepackage{makecell}
\usepackage{multirow}
\usepackage{threeparttable}
\usepackage{pgfplots}
\usepackage{subcaption}
\usepackage{titletoc}
\pgfplotsset{compat=1.18}

\usepackage{amsthm}

\theoremstyle{definition}

\usepackage[utf8]{inputenc} %
\usepackage[T1]{fontenc}    %
\usepackage{hyperref}       %
\usepackage{url}            %
\usepackage{amsfonts}       %
\usepackage{nicefrac}       %
\usepackage{caption}
\newcommand{\metricci}[2]{\makecell[r]{#1\\[-1pt]{\scriptsize$\pm\,#2$}}}

\iclrfinalcopy

\newtheorem{theorem}{Theorem}[section]

\newtheorem{lemma}[theorem]{Lemma}
\newtheorem{proposition}{Proposition}
\theoremstyle{remark}
\newtheorem*{remark}{Remark}

\definecolor{questioncolor}{HTML}{EAF0FB}   %
\definecolor{teachercolor}{HTML}{E8F5E9}    %
\definecolor{poecolor}{HTML}{FFF8E1}        %
\definecolor{proxycolor}{HTML}{FFEBEE}      %

\newtcolorbox{questionbox}{
  colback=questioncolor, colframe=questioncolor!60!black,
  fonttitle=\bfseries, title=Question,
  boxrule=0.4pt, arc=2pt, left=6pt, right=6pt, top=4pt, bottom=4pt
}

\newtcolorbox{tracebox}[1]{
  colback=white, colframe=black!40,
  fonttitle=\bfseries\small, title={#1},
  boxrule=0.4pt, arc=2pt, left=6pt, right=6pt, top=4pt, bottom=4pt
}

\newtcolorbox{teacherbox}[1]{
  colback=teachercolor, colframe=teachercolor!50!black,
  fonttitle=\bfseries\small, title={#1},
  boxrule=0.4pt, arc=2pt, left=6pt, right=6pt, top=4pt, bottom=4pt
}
\newtcolorbox{poebox}[1]{
  colback=poecolor, colframe=poecolor!50!black,
  fonttitle=\bfseries\small, title={#1},
  boxrule=0.4pt, arc=2pt, left=6pt, right=6pt, top=4pt, bottom=4pt
}
\newtcolorbox{proxybox}[1]{
  colback=proxycolor, colframe=proxycolor!50!black,
  fonttitle=\bfseries\small, title={#1},
  boxrule=0.4pt, arc=2pt, left=6pt, right=6pt, top=4pt, bottom=4pt
}

\author{%
  Youssef Allouah\thanks{Equal contribution. Alphabetical order. Work done  at the Simons Institute for the Theory of Computing, UC Berkeley.} \\
  Stanford University\\
  \texttt{yallouah@cs.stanford.edu} \\
  \And
  Mahdi Haghifam\footnotemark[1] \\
  Toyota Technological Institute at Chicago \\
  \texttt{mhaghifam@ttic.edu} \\
  \AND
  Sanmi Koyejo\thanks{Equal advising. Alphabetical order.} \\
    Stanford University \\
  \texttt{sanmi@cs.stanford.edu}$\quad\quad\quad\quad\,\,\,\,$\\
  \And 
  Reza Shokri\footnotemark[2]  \\
  National University of Singapore\\
  \texttt{reza@comp.nus.edu.sg} \\
}

\title{The Distillation Game:\\
Adaptive Evaluations \& Efficient Defenses}

\begin{document}
\maketitle

\vspace{-7mm}
\begin{abstract}
Distillation attacks create a deployment trade-off for model providers: the same outputs that make a model more useful can also make it easier to imitate. 
We study this trade-off through a minimax game between a utility-constrained teacher and an adaptive student. Our framework yields tractable one-sided response rules: an adaptive evaluation rule in which the student reweights high-value examples, and a teacher-side defense template that suppresses outputs most useful for distillation.
From a cheap proxy for example value, we derive Product-of-Experts (PoE), a simple forward-pass-only defense that combines the teacher with a proxy student during generation.
Empirically, adaptive evaluation reveals a large passive--adaptive gap: on state-of-the-art defenses, adaptive students recover substantially more capability than passive evaluation suggests on GSM8K and MATH. Under this stronger evaluation, the apparent robustness gap between expensive defenses and PoE narrows considerably, while PoE remains substantially cheaper and preserves higher-quality reasoning traces. Overall, our results suggest that strong distillation remains difficult to stop, and that progress on antidistillation should be judged against adaptive students rather than passive ones.
Our code is available at: \url{https://github.com/ysfalh/distillation-game}.
\end{abstract}

\vspace{-4mm}
\section{Introduction}

As model providers expose richer outputs, they also expose more reusable training signal for distillation attacks. Any interface that reveals model capabilities---through answers, intermediate reasoning, tool-use sequences, error corrections, feedback on candidate solutions, or other multi-step interactions---can in principle be aggregated into a dataset for imitation. This creates a real deployment trade-off: the same richness that makes a system more useful, transparent, and interactive can also make its capabilities easier to distill. Reasoning traces are a particularly visible example because they expose intermediate structure explicitly, but the underlying issue is broader than chain-of-thought alone. This matters not only because distillation can copy capabilities that were costly to build, but also because it can accelerate the diffusion of safety-relevant capabilities~\citep{trockman2026antidistillation}. Scientifically, this makes distillation a useful lens for studying which exposed model behaviors contain reusable learning signal.

Recent work has begun to study defenses against distillation, a.k.a.\ \emph{antidistillation}, by modifying the outputs a teacher reveals, especially in reasoning-oriented settings~\citep{savani2025antidistillation,li2025doge,ding2025information,zheng2025model,ma2026protecting}. Across this literature, defense and evaluation are typically studied in isolation: a defense proposes a rule for modifying teacher outputs, 
and is then evaluated against a distiller that trains uniformly on the released data.
This separation leaves two questions unanswered. First, what kind of distiller should the defense be judged against? A realistic attacker need not train uniformly on released data; after observing defended outputs, it can, for example, filter, reweight, or concentrate on the examples that carry the greatest learning value. 
Without accounting for this adaptation, evaluations risk a cat-and-mouse cycle in which each new defense is tested only against the weak distiller it implicitly assumes~\citep{athalye2018obfuscated}. Second, different defenses encode different implicit theories of what makes an output useful to a distiller, but no common framework has been available to compare them or to ask which notion of usefulness a defense should target. We address both questions with a single minimax framework: its attacker-side best response gives a principled adaptive evaluation rule, and its defender-side best response unifies existing methods and enables designing new ones like our PoE defense, which is designed to be cheap and also preserves trace quality better.

Distillation attacks are not just a theoretical concern. Frontier model providers already limit the information exposed through reasoning-oriented interfaces, for example by hiding full chain-of-thought or replacing it with shorter summaries~\citep{google2025thinking,openai2025reasoning,anthropic2025thinking}. Public reporting also suggests that distillation attacks are a live operational issue rather than a hypothetical one~\citep{anthropic2026distillation}. 
But limiting explicit traces does not resolve the underlying problem, because the useful signal available to a distiller is not confined to verbatim chain-of-thought. In our experiments with commercial frontier-model summaries (Appendix~\ref{sec:add-expr} and Fig.~\ref{fig:frontier}), even passive students recover nontrivial capability, suggesting that summarization alone should not be assumed to eliminate distillation risk. More generally, once some useful signal is released, a realistic distiller may further filter or reweight the data it observes, motivating evaluation against adaptive students permitted by the threat model rather than only passive ones.

\vspace{-2mm}
\subsection{Contributions}
\vspace{-2mm}
Our main contribution is a game-theoretic framework for distillation attacks and defenses. We model the teacher as choosing a released model $\pi_{\mathrm{rel}}$ that remains close to a reference teacher $\pi_{\mathrm{ref}}$, while the student chooses an effective distribution $\pi_{\mathrm{eff}}$ that remains close to the released model but can adaptively concentrate on more useful outputs. Writing $x$ for a prompt and $y$ for a corresponding output, let $v(x,y)$ denote the value of that prompt-output pair for downstream distillation. This leads to the minimax objective:
\begin{equation}
\mathscr{V}(\varepsilon,\rho)
:=
\inf_{\pi_{\mathrm{rel}}\in \Pi_\varepsilon(\pi_{\mathrm{ref}})}
\sup_{\pi_{\mathrm{eff}}\in \Pi_\rho(\pi_{\mathrm{rel}})}
\mathbb{E}_{x\sim \mathscr{D},\; y\sim \pi_{\mathrm{eff}}(\cdot \mid x)}[v(x,y)].
\end{equation}
Above, $\mathscr{D}$ is the prompt distribution, $\Pi_\varepsilon(\pi_{\mathrm{ref}})$ is the set of released models within the teacher's fidelity budget $\varepsilon$, and $\Pi_\rho(\pi_{\mathrm{rel}})$ is the set of effective student distributions within the adaptation budget $\rho$; Section~\ref{sec:problem-statement} gives the precise KL-based definitions. This objective is central because, once $v$ and the budgets are specified, it yields explicit best responses for adaptive evaluation on the student side and a corresponding template for teacher-side defense.
Our contributions include:
\begin{itemize}
    \item \textbf{Game formulation.} We formulate distillation defense as a minimax game between a utility-constrained teacher and an adaptive student. This gives a common framework for specifying adaptive evaluation protocols and deriving teacher-side defense rules. By making the attacker and defender explicit in one tractable model, it provides a methodological foundation for designing and evaluating future distillation defenses against adaptive students.

    \item \textbf{Adaptive evaluation and attack.} We derive the student's best response in our framework, which concentrates on examples with higher learning value and yields a principled adaptive evaluation rule. Empirically, this exposes a large adaptive-passive gap on both GSM8K and MATH: for state-of-the-art defenses, adaptive evaluation increases downstream student accuracy by roughly $50\%$ relative to passive evaluation. 

    \item \textbf{Practical  defense.} We instantiate our teacher defense principle with a cheap proxy for learning value and derive Product-of-Experts (PoE), a gradient-free generation rule that combines the teacher with a proxy student.
    Under adaptive evaluation, the apparent robustness gap between stronger expensive defenses and PoE narrows, while PoE remains substantially cheaper to run; on GSM8K, for example, it incurs about $1.6\times$ teacher-time overhead rather than $2.9\times$ for the state-of-the-art. 
    Besides being efficient, this proxy induces a conservative defense bias: it suppresses outputs with large teacher-over-student likelihood gaps while staying close to outputs the teacher itself finds plausible.
    Indeed, our rubric-based LLM judge evaluation suggests that PoE produces higher-quality traces.
\end{itemize}

The main methodological takeaway from our work is that antidistillation defenses should be evaluated jointly on \textit{(i)} leakage to adaptive students, \textit{(ii)} generation cost, and \textit{(iii)} the quality of the released traces, rather than by the accuracy trade-off against a passive student alone.
Indeed, a realistic distiller need not train passively on the data it receives. Even with arguably the simplest form of adaptivity--- reweighting the released traces without issuing new queries---defenses that look strong against a passive student can leak much more.
Since any richer distillation attack can also afford to reweight the same release, the leakage we report is a \emph{lower bound}.
Once cost and trace quality are also taken into account, no single defense currently dominates as our results on PoE defense show.

We deliberately follow the setting of \cite{savani2025antidistillation} which is an important early work in this emerging area, including its focus on mathematical reasoning, so that the effect of attacker adaptation can be isolated cleanly.
In doing so, our goal is mainly to establish stronger evaluation foundations for this area.
Appendix~\ref{app:ablations} also examines amother defense based on trace rewriting~\citep{ma2026protecting} where we show it can occasion a significant loss in trace quality.

\begin{figure}
\vspace{-12mm}
      \centering
\includegraphics[width=0.78\textwidth,height=0.22\textheight]{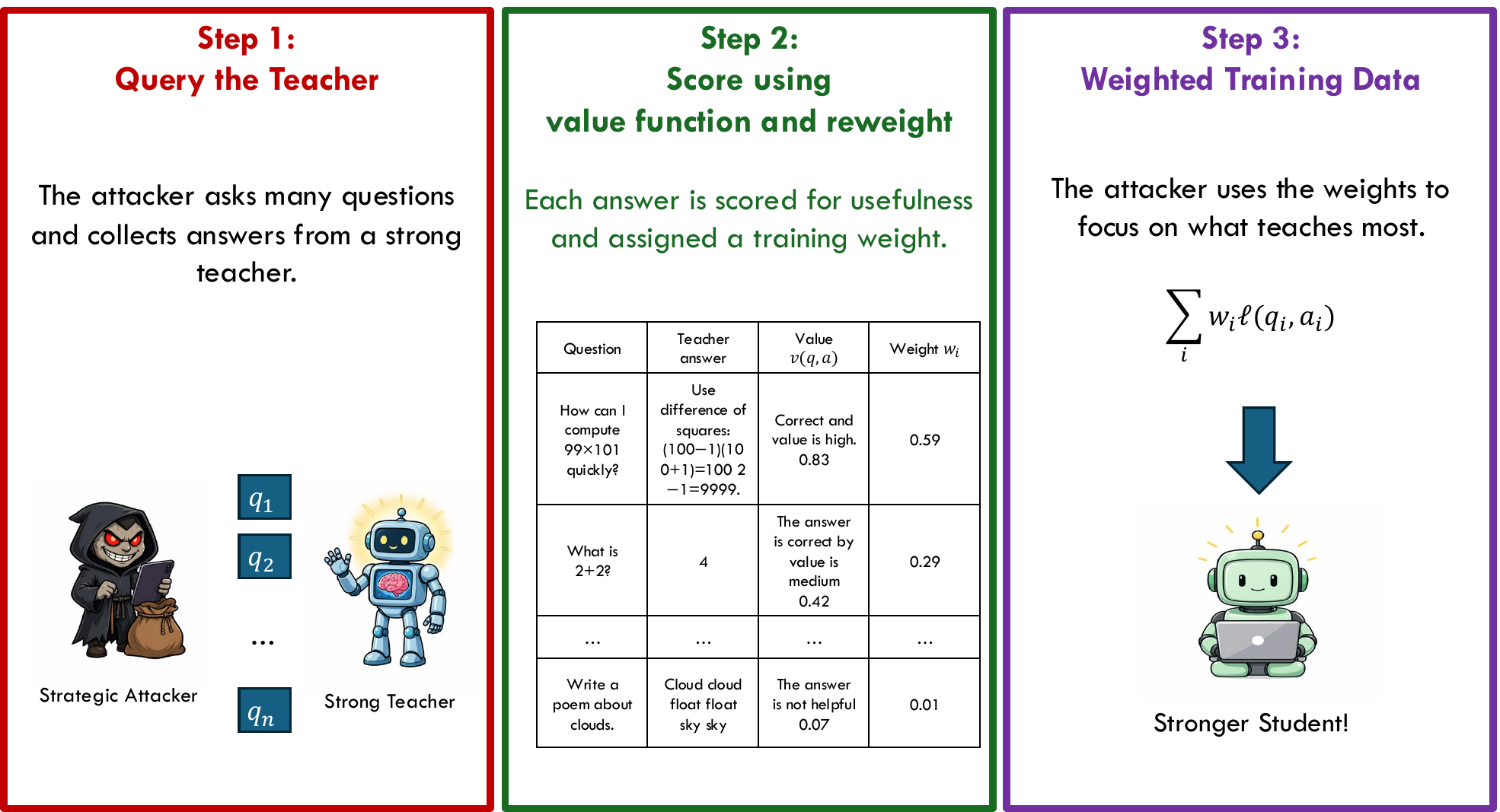}
    \caption{An adaptive attacker does not train uniformly on all teacher outputs; it estimates the usefulness of each queried sample and assigns larger training weight to higher-value responses.}
    \label{fig:strategic-cartoon}
    \vspace{-7mm}
\end{figure}

\vspace{-3mm}
\subsection{Related Work}
\vspace{-2mm}
\textbf{Defenses against distillation.}
Our work is closest to \citet{savani2025antidistillation}, who propose antidistillation sampling (ADS), an inference-time method that modifies the teacher's next-token distribution using finite-difference gradient estimates through a proxy model. Our framework provides a game-theoretic foundation for their approach: ADS arises as the teacher's best response in our minimax formulation under a specific approximation. Crucially, our adaptive student evaluation reveals that ADS is substantially less effective against an adaptive distiller than passive evaluation suggests. There is another line of work that aims not to suppress distillation directly, but to detect or attribute it after the fact \citep{xu2026antidistillation,kirchenbauer2023watermark}.
\citet{li2025doge} and \citet{fang2026towards} instead fine-tune the teacher model itself to defend against distillation. DOGe~\citep{li2025doge} adversarially trains only the final linear layer to maximize KL divergence from a proxy student, while \citet{fang2026towards} learn a logit transformation matrix guided by a conditional mutual information objective. 
Neither work evaluates against an adaptive distillation attack.

\textbf{Trace rewriting and inversion.}
Another line of work modifies the content of reasoning traces after generation to degrade their usefulness for distillation~\citep{ma2026protecting,ding2025information}. This approach is orthogonal to our decoding-time framework: one could first apply PoE sampling and then rewrite the resulting traces. Again, neither work evaluates against a student that reweights traces by learning value, which is the central threat our formulation addresses.
\citet{zhang2026steal} introduce trace inversion, where a separate model is trained to reconstruct detailed chain-of-thought from the summaries and answers exposed by a target model, and the resulting synthetic traces are used for student fine-tuning. Our adaptive student addresses the same goal but through the far simpler mechanism of reweighting existing traces.

\textbf{Decoding-time combinations.} Our PoE rule belongs to a broader family of decoding-time methods that combine language models at the logit level. DExperts~\citep{liu2021dexperts} generates from a geometric combination of a base model with expert and anti-expert models to steer toward desired attributes, and Contrastive Decoding~\citep{li2023contrastive} subtracts a small model's log-likelihood from a large one's to amplify the capability gap and improve generation quality. 

In concurrent work, \cite{libon2026does} argue that distillation defenses should be evaluated under explicit threat models characterized by query and data budgets, and interface profile, and use ADS \citep{savani2025antidistillation} to show that conclusions can change with these assumptions. Our focus is complementary: holding the queried prompt set and interface fixed, we study a deliberately weak form of post-release adaptation in which the attacker only reweights traces already obtained, yielding both a tractable adaptive evaluation rule and a corresponding teacher-side defense template.

\vspace{-2mm}
\section{Problem Statement}
\label{sec:problem-statement}
\vspace{-2mm}

We now formalize the distillation game between a teacher and a student. We first define the models and their allowable deviations, then introduce a value function that measures how useful an output is for downstream distillation, and finally combine these ingredients into a minimax objective.

\textbf{Setup and models.}
Let $\mathscr{X}$ denote the space of contexts, such as user prompts, drawn from a distribution $\mathscr{D}$, and let $\mathscr{Y}$ denote the space of outputs, such as reasoning traces and final answers, and $\Delta(\mathscr{Y})$ the set of probability distributions over $\mathscr{Y}$. A language model is represented by a conditional distribution $\pi(\cdot \mid x) \in \Delta(\mathscr{Y})$ over outputs given a context $x \in \mathscr{X}$. We write $\mathscr{Z} = \mathscr{X} \times \mathscr{Y}$ and denote examples by $z=(x,y)\in \mathscr{Z}$.
Our main formulation involves three objects:
\begin{enumerate*}[label=(\roman*), itemjoin={{; }}, itemjoin*={{; and }}]
    \item \emph{Reference teacher} $\pi_{\mathrm{ref}}$: the undefended teacher model
    \item \emph{Released teacher} $\pi_{\mathrm{rel}}$: the defended model actually exposed to users
     \item \emph{effective student distribution} $\pi_{\mathrm{eff}}$: the distribution induced by the student's reweighting of released outputs.
\end{enumerate*}
Here $\pi_{\mathrm{eff}}$ is not the (final) student model $\pi_{\mathrm{stu}}$, but the effective training distribution seen by the student.

\vspace{-2mm}
\subsection{Teacher fidelity and student adaptation budgets}
\vspace{-1mm}

The released teacher should remain close to the reference teacher in order to preserve output quality. 
To formalize it, we use KL divergence to capture closeness as is customary in the alignment literature (see \cite{schulman2017proximal,rafailov2023direct}).
We then define the teacher fidelity set:
\begin{equation}
\Pi_\varepsilon(\pi_{\mathrm{ref}})
:=
\left\{
\pi  \in \Delta(\mathscr{Y})  :
\mathbb{E}_{x\sim \mathscr{D}}\!\left[
\mathrm{D}_{\mathrm{KL}}\!\bigl(\pi(\cdot\mid x)\,\|\,\pi_{\mathrm{ref}}(\cdot\mid x)\bigr)
\right]
\le \varepsilon
\right\}.
\end{equation}

The budget $\varepsilon \ge 0$ controls how much the released teacher may depart from the reference teacher. This direction of KL penalizes the released model for placing mass on outputs that the reference teacher considers unlikely, so it naturally captures a fidelity constraint.

A passive student trains uniformly on samples from the released teacher. A stronger student may instead filter, reweight, or subsample those samples, thereby inducing an effective training distribution. We model this by requiring this effective distribution to remain close to the released teacher:
\begin{equation}
\Pi_\rho(\pi_{\mathrm{rel}})
:=
\left\{
\pi \in \Delta(\mathscr{Y})  :
\mathbb{E}_{x\sim \mathscr{D}}\!\left[
\mathrm{D}_{\mathrm{KL}}\!\bigl(\pi(\cdot\mid x)\,\|\,\pi_{\mathrm{rel}}(\cdot\mid x)\bigr)
\right]
\le \rho
\right\}.
\end{equation}

Here $\rho \ge 0$ is the student adaptation budget. When $\rho = 0$, the student is passive and trains on the released distribution as given. Larger values of $\rho$ allow progressively more selective concentration on high-value outputs. The intuition for this restriction is that the attacker cannot invent an entirely new data source. Rather, we want to model a practical attacker that can only filter, reweight, or subsample outputs that were actually released. 
Indeed, the attacker chooses a change of measure that is absolutely continuous with respect to the released distribution.

\vspace{-1mm}
\subsection{Value function}
\vspace{-2mm}

The last ingredient is a scalar value function
$
v : \mathscr{X} \times \mathscr{Y} \to \mathbb{R}
$
that measures how useful a released example $(x,y)$ is for  distillation.
Intuitively, $v(x,y)$ is large if and only if training on $(x,y)$ helps the student more on the downstream task.
Our game formulation only
requires such a scalar notion of usefulness; it does not depend on any particular student
update rule or parameterization.

A natural first-order proxy is obtained by asking how upweighting an example changes a
downstream objective. Let $\theta$ denote the student parameters, initialized at $\theta_0$,
and let $\mathscr{L}(\theta)$ be the downstream task loss at $\theta$. If we locally approximate student training by a
gradient step on the proxy student's negative log-likelihood, then the value of $(x,y)$ is
\begin{equation}
v_{\mathrm{grad}}(x,y)
:=
-\left\langle
\nabla_\theta \mathscr{L}(\theta_0),
\nabla_\theta \log {\pi}_{\mathrm{stu}}(y \mid x;\theta)\big|_{\theta=\theta_0}
\right\rangle.
\end{equation}

We use $v_{\mathrm{grad}}$ here only as an illustrative instance. The game and best-response
analysis in the next section apply to any choice of $v$. Later, in Section~\ref{sec:poe}, we
introduce a cheaper proxy value based on the log-likelihood ratio, which leads to a practical defense.

\begin{remark}[Proxy students]
In practice, a defender need not know the attacker's exact model. It can optimize against a proxy model
$\hat{\pi}_{\mathsf{stu}}$, typically chosen to be substantially cheaper than the
teacher, with the goal of approximating which outputs are highly informative for distillation. This proxy-based
approximation is a common design pattern in prior work on distillation defenses~\citep{savani2025antidistillation,li2025doge}, and is also used in our experiments.
\end{remark}

\vspace{-2mm}
\subsection{The distillation game}
\vspace{-2mm}

We now combine the previous ingredients into a single minimax objective. Given a value function $v$, the teacher chooses a released model within its fidelity budget, while the student chooses an effective training distribution within its adaptation budget:
\begin{equation} \label{eq:game}
\mathscr{V}(\varepsilon, \rho)
:=
\inf_{\pi_{\mathrm{rel}} \in \Pi_\varepsilon(\pi_{\mathrm{ref}})}
\sup_{\pi_{\mathrm{eff}} \in \Pi_\rho(\pi_{\mathrm{rel}})}
\mathbb{E}_{x \sim \mathscr{D},\; y \sim \pi_{\mathrm{eff}}(\cdot \mid x)}
\bigl[v(x,y)\bigr].
\end{equation}

The inner maximization identifies the strongest student response to a fixed released model: among all effective distributions that remain close to $\pi_{\mathrm{rel}}$, the student chooses the one with largest expected value. The outer minimization identifies the teacher's best defense: among all released models that remain close to the reference teacher, it chooses the one that minimizes the student's attainable value.

This formulation separates the two design choices in the problem. The budgets $(\varepsilon,\rho)$ specify the threat model, while the value function $v$ specifies what makes an output useful for distillation. In the next section, we show that for fixed $v$, both the student and teacher best responses admit simple closed forms. Different choices of $v$ then lead to different attack and defense rules.

\vspace{-1mm}
\begin{algorithm}[t]
\caption{Adaptive distillation  attack with gradient-based value function}
\label{alg:adaptive-student}
\begin{algorithmic}[1]
\Require Traces $\mathscr{T} = \{(x^{(i)}, y^{(i)})\}_{i=1}^{m}$ sampled from $\pi_{\mathsf{rel}}$, student init $\theta_0$, stepsize $\alpha$, sharpness $\eta$
\Require Student downstream loss $\mathscr{L}(\theta)$ (e.g., NLL on a held-out validation set)
\For{each training step / minibatch $B \subset \mathscr{T}$}
    \State $g \gets \nabla_\theta \mathscr{L}(\theta)$ \Comment{Downstream task gradient}
    \For{each trace $(x, y) \in B$}
        \State $v_{\text{grad}}(x, y) \gets -\left\langle g,\; \nabla_\theta \log \pi_{\mathsf{stu}}(y \mid x; \theta) \right\rangle$
       
        \State $w(x, y) \gets \frac{\exp\bigl(\eta \cdot v_{\text{grad}}(x, y)\bigr)}{\sum_{(x', y') \in B} \exp\bigl(\eta \cdot v_{\text{grad}}(x', y')\bigr)}$ \Comment{Normalize weights over the batch}
    \EndFor
    \State $\theta \gets \theta - \alpha\, \nabla_\theta \Bigl[\sum_{(x,y) \in B} w(x,y) \sum_{t} -\log p(y_{t+1} \mid y_{1:t}, x;\, \theta) \Bigr]$
\EndFor
\State \Return $\theta$
\end{algorithmic}
\end{algorithm}

\section{Best Responses and Implications}
\vspace{-1mm}
We now show that the student's response to a released teacher and the teacher's response to a fixed student value take exponential-tilt forms. These best responses are also useful algorithmically: they directly give an adaptive evaluation rule for the student and a family of teacher-side defenses. We state the result first, whose proof is in Appendix \ref{appx:proof}, and then unpack its two main implications. 

\begin{restatable}[Best responses]{theorem}{bestresponsesthm}
\label{thm:best-responses}
Assume $\mathscr{Y}$ is finite.
For value function $v$, the student's best response to a fixed released policy $\pi_{\mathsf{rel}}$ and the teacher's best response to a fixed student are respectively, for every $y \in \mathscr{Y}$ and $x \in \mathscr{X}$:
\[
\pi_{\mathsf{eff}}^\star(y \mid x) \propto \pi_{\mathsf{rel}}(y \mid x)\, e^{\,\eta\, v(x,y)}, \qquad \pi_{\mathsf{rel}}^\star(y \mid x) \propto \pi_{\mathsf{ref}}(y \mid x)\, e^{-\lambda\, v(x,y)}
\]
where $\eta, \lambda \geq 0$ are dual variables. When the global KL constraints are active, $\eta$ and $\lambda$ are unique constants chosen such that the expected budgets are tight, i.e., $\mathbb{E}_{x \sim \mathscr{D}} [\mathrm{D}_{\mathrm{KL}}(\pi_{\mathsf{eff}}^\star(\cdot \mid x) \| \pi_{\mathsf{rel}}(\cdot \mid x))] = \rho$ and $\mathbb{E}_{x \sim \mathscr{D}} [\mathrm{D}_{\mathrm{KL}}(\pi_{\mathsf{rel}}^\star(\cdot \mid x) \| \pi_{\mathsf{ref}}(\cdot \mid x))] = \varepsilon$.
\end{restatable} 
Theorem~\ref{thm:best-responses} above characterizes the two one-sided best responses, which are what our adaptive evaluation rule and defense template use. In Appendix~\ref{app:stackelberg}, we complement this with an analysis of the full Stackelberg game~\eqref{eq:game}, proving the existence of an optimal teacher--student pair.

\begin{figure}[t]
  \centering
  \tiny
  \begin{tcbraster}[raster columns=2, raster equal height, raster column skip=6pt]
    \begin{proxybox}{High-value trace}
      \raggedright
      I'm trying to figure out how many miles Jack has driven over the past 9
      years. He drives 37,000 miles every four months. First, I need to find
      out how many sets of four months are in 9 years. I know there are 12
      months in a year, so in 9 years, there are $9 \times 12 = 108$ months.
      $\cdots$
      There are 3 four-month periods in a year, so he drives
      $37{,}000 \times 3 = 111{,}000$ miles per year. Multiplying by 9 years
      gives $111{,}000 \times 9 = 999{,}000$ miles.
      \textbf{Final Answer:} $\boxed{999000}$
    \end{proxybox}
    \begin{proxybox}{Low-value trace}
      \raggedright
      First, I recognize that the spaceship's speed is halved for every
      additional 100 people beyond the initial 200.
      \centerline{\texttt{raiding raiding raiding raiding raiding \dots}}
      \centerline{\texttt{raiding raiding raiding raiding raiding \dots}}
      \centerline{$\vdots$}
      \centerline{\texttt{raiding raiding raiding raiding raiding \dots}}
    \end{proxybox}
  \end{tcbraster}
  \caption{An adaptive student down-weights traces with low downstream
    gradient alignment $v_{\mathrm{grad}}$. High-value traces contain useful
    reasoning; low-value traces are repetitive or uninformative.}
  \label{fig:strategic_student_filter}
  \vspace{-1em}
\end{figure}

\subsection{First Implication: Adaptive Student}
The student's best response in Theorem~\ref{thm:best-responses} is an exponential tilt toward high-value examples. For a fixed released teacher $\pi_{\mathrm{rel}}$, the effective student distribution satisfies:
\begin{equation}
\pi_{\mathrm{eff}}^\star(y \mid x)
\propto
\pi_{\mathrm{rel}}(y \mid x)\exp\{\eta\, v(x,y)\}.
\end{equation}

This gives a principled adaptive evaluation rule. Rather than training uniformly on released traces as given, the student first reweights them exponentially according to their value and then trains on the resulting weighted dataset. The sharpness parameter $\eta$ controls how selective this reweighting is: $\eta = 0$ recovers passive training, while larger values of $\eta$ place progressively more mass on high-value traces.
Algorithm~\ref{alg:adaptive-student} implements this idea as a stochastic minibatch procedure. The population rule in Theorem~\ref{thm:best-responses} defines an exponential tilt over the released distribution, while the implementation estimates value scores on each minibatch and normalizes weights within that minibatch.

\begin{algorithm}[t]
\caption{Product-of-Experts (PoE) Sampling}
\label{alg:teacher-poe}
\begin{algorithmic}[1]
\Require Prompt $x_{1:n}$, max tokens $N$, teacher model $\theta_T$, proxy student model $\theta_S$, mixture weight $\gamma \in (0, 1)$, temperature $\tau$
\For{$t = n, n+1, \dots, N-1$}
  \State  Compute teacher log-probabilities: 
         $z_T(\cdot) \leftarrow \log p(\cdot \mid x_{1:t}; \theta_T)$
  \State  Compute proxy log-probabilities: 
         $z_S(\cdot) \leftarrow \log p(\cdot \mid x_{1:t}; \theta_S)$
  \State Draw next token from the combined geometric mixture distribution:
  \State $x_{t+1} \propto \exp\!\left( \frac{1}{\tau} \Big[ (1-\gamma) z_T(\cdot) + \gamma z_S(\cdot) \Big] \right)$
\EndFor
\State \Return $x_{1:N}$
\end{algorithmic}
\end{algorithm}
\subsection{Second Implication: Teacher-Side Defenses}

\label{sec:poe}

The teacher's best response in Theorem~\ref{thm:best-responses} shows that, once a value function $v$ has been specified, the released teacher should tilt away from outputs with high value to the student. For a fixed value function and letting $\lambda \ge 0$ control the strength of the defense, the best-response rule takes the form
\begin{equation}
\pi_{\mathrm{rel}}^\star(y \mid x)
\propto
\pi_{\mathrm{ref}}(y \mid x)\exp\{-\lambda v(x,y)\}.
\end{equation}
This turns any choice of $v$ into a corresponding teacher-side defense.
A first instantiation uses the gradient-based value from Section~2.3,
\begin{equation}
v_{\mathrm{grad}}(x,y)
=
-\left\langle
\nabla_\theta \mathscr{L}(\theta_0),
\nabla_\theta \log {\pi}_{\mathrm{stu}}(y \mid x; \theta)\big|_{\theta=\theta_0}
\right\rangle.
\end{equation}
Substituting this into the teacher best response gives
\begin{equation}
\pi_{\mathrm{rel}}^\star(y \mid x)
\propto
\pi_{\mathrm{ref}}(y \mid x)
\exp\!\left\{
\lambda
\left\langle
\nabla_\theta \mathscr{L}(\theta_0),
\nabla_\theta \log {\pi}_{\mathrm{stu}}(y \mid x; \theta)\big|_{\theta=\theta_0}
\right\rangle
\right\}.
\end{equation}
This captures the principle behind the antidistillation sampling (ADS) method by~\cite{savani2025antidistillation}: the teacher suppresses outputs whose student gradient is most aligned with reducing downstream risk. Thus, our framework gives a game-theoretic interpretation of the objective that ADS approximates under a first-order gradient-based notion of example value.

The main drawback of the above method is computational cost. Evaluating $v_{\mathrm{grad}}$ during decoding requires gradient information from the student, which is expensive in autoregressive generation. This motivates a cheaper proxy value function that avoids gradient computation altogether:
\begin{equation}
v_{\mathrm{gap}}(x,y)
:=
\log \pi_{\mathrm{ref}}(y \mid x)
-
\log {\pi}_{\mathrm{stu}}(y \mid x).
\end{equation}

This quantity is large when the teacher assigns much higher likelihood than the proxy student to an output. Such outputs mark regions where the teacher has capability or confidence not yet matched by the proxy, and are therefore plausible carriers of distillation value.
Substituting $v_{\mathrm{gap}}$ into the same best-response rule gives:
\begin{equation}
\pi_{\mathrm{rel}}^\star(y \mid x)
\propto
\pi_{\mathrm{ref}}(y \mid x)
\exp\!\left\{
-\lambda\bigl(\log \pi_{\mathrm{ref}}(y \mid x)-\log {\pi}_{\mathrm{stu}}(y \mid x)\bigr)
\right\},
\end{equation}
which simplifies to:
\begin{equation}
\pi_{\mathrm{rel}}^\star(y \mid x)
\propto
\pi_{\mathrm{ref}}(y \mid x)^{1-\lambda}
{\pi}_{\mathrm{stu}}(y \mid x)^\lambda.
\end{equation}

This is a geometric mixture of the teacher and student, which we refer to as a Product-of-Experts (PoE). Intuitively, the defense downweights outputs on which the teacher has a large likelihood advantage over the student.
The name follows the work of \cite{hinton2002training}, where multiple probabilistic models are combined by multiplying their densities. Here, the product form arises as a proxy teacher response in our game, rather than as a generative modeling assumption.
In Algorithm~\ref{alg:teacher-poe} and henceforth, we write the mixture weight as $\gamma \in [0,1]$, to differentiate from the ADS parameter $\lambda$.

In practice, PoE is cheap because it is forward-pass-only: it combines teacher and (proxy) student log-probabilities during generation and requires no gradient estimates through the student.
The ideal rule is defined over complete output sequences, but exact sequence-level sampling is intractable for autoregressive language models. We therefore use a standard token-level approximation, interpolating the teacher and proxy-student predictions at each decoding step, as is common for inference-time methods~\citep{yang2021fudge,krause2021gedi,li2023contrastive,savani2025antidistillation}. 

Unlike the earlier gradient-based defenses, this rule does not directly optimize for making the student's training update unhelpful; instead, it regularizes the released policy toward outputs that remain plausible under both models. This suggests a qualitative advantage that we test empirically in Section~\ref{sec:experiments}: because PoE preferentially retains continuations supported by both the teacher and proxy student, it may distort reasoning traces less severely than gradient-based shaping.

\vspace{-2mm}
\section{Empirical Results}
\vspace{-2mm}
\label{sec:experiments}

We evaluate three questions. First, does adaptive evaluation reveal more distillation leakage than passive evaluation? Second, how do ADS and PoE compare along the utility--distillability frontier, where teacher accuracy measures utility and student accuracy measures leakage? Third, how does this comparison change once runtime and trace quality are taken into account?

\renewcommand{\metricci}[2]{$#1_{\pm #2}$}
\newcommand{\hp}[1]{{\scriptsize #1}}
\newcommand{\vrulegap}{\hspace{4pt}\vline\hspace{4pt}}

\begin{table}[t!]
\vspace{-5mm}
\centering
\small
\setlength{\tabcolsep}{2pt}
\renewcommand{\arraystretch}{1.25}
\begin{threeparttable}
\makebox[\textwidth][c]{%
\begin{tabular}{@{}ll@{\vrulegap}rrr@{\vrulegap}cc@{}}
\toprule
\textbf{Dataset} &
\textbf{Teacher} &
\multicolumn{1}{c}{\textbf{Teacher Acc. $\uparrow$}} &
\multicolumn{1}{c}{\textbf{Passive Acc. $\downarrow$}} &
\multicolumn{1}{c@{\vrulegap}}{\textbf{Adaptive$^\dagger$ Acc. $\downarrow$}} &
\multicolumn{1}{c}{\textbf{Rel. Gain}} &
\multicolumn{1}{c}{\textbf{Time Cost}} \\
\midrule
\addlinespace[3pt]
\multirow{3}{*}{\textbf{\texttt{GSM8K}}}
& Standard
& \metricci{87.22\%}{0.04}
& \metricci{57.24\%}{0.25}
& \metricci{56.74\%}{0.17}
& $-0.87\%$
& $1.00\times$ \\

& ADS \hp{($\lambda=0.052$)}
& \metricci{82.13\%}{0.43}
& \metricci{34.33\%}{0.17}
& \metricci{51.50\%}{1.46}
& $50.04\%$
& $2.93\times$ \\

& PoE$^\dagger$ \hp{($\gamma=0.65$)}
& \metricci{81.61\%}{0.46}
& \metricci{39.26\%}{3.33}
& \metricci{49.46\%}{1.19}
& $25.98\%$
& $1.64\times$ \\

\midrule
\addlinespace[3pt]
\multirow{3}{*}{\textbf{\texttt{MATH}}}
& Standard
& \metricci{61.78\%}{0.33}
& \metricci{15.17\%}{0.29}
& \metricci{15.29\%}{0.40}
& $0.75\%$
& $1.00\times$ \\

& ADS \hp{($\lambda=0.08$)}
& \metricci{61.16\%}{0.36}
& \metricci{8.96\%}{1.30}
& \metricci{13.45\%}{0.99}
& $50.07\%$
& $3.85\times$ \\

& PoE$^\dagger$ \hp{($\gamma=0.75$)}
& \metricci{60.07\%}{0.48}
& \metricci{9.00\%}{2.86}
& \metricci{12.92\%}{1.13}
& $43.56\%$
& $2.33\times$ \\
\bottomrule
\end{tabular}%
}
\end{threeparttable}
\vspace{2mm}
\caption{Representative points on the utility--distillability frontier. $^\dagger$ denotes our method/evaluation. Passive and adaptive$^\dagger$ columns report student accuracy after distillation larger values indicate greater leakage. Rel. gain is the relative improvement from adaptive evaluation. Teachers include the standard model, ADS~\citep{savani2025antidistillation}, and our PoE$^\dagger$ defense. Time cost is the generation time overhead relative to the standard teacher. \texttt{GSM8K} shows the clearest separation, while \texttt{MATH} shows the same qualitative trend with higher variance. Entries report standard error of the mean over 3 seeds.}
\label{tab:seed_aggregates}
\vspace{-1em}
\end{table}

\vspace{-1mm}
\subsection{Experimental setup}
\vspace{-1mm}

We report our main experimental setup below, and defer full details to Appendix~\ref{app:experimental-details}.

\textbf{Datasets and models.}
We evaluate on \texttt{GSM8K}~\citep{cobbe2021training} and \texttt{MATH}~\citep{hendrycks2021measuring}, two standard benchmarks for mathematical reasoning which provide a high signal for distillation (see Appendix~\ref{app:distillation_signal}), and require multi-step reasoning for strong performance.
Our teacher model is \texttt{DeepSeek-R1-Distill-Qwen-7B}~\citep{guo2025deepseek}. We use \texttt{Qwen2.5-3B}~\citep{bai2023qwen} as the proxy student, and \texttt{Llama-3.2-3B}~\citep{grattafiori2024llama} as the final student. 

\textbf{Teacher types.}
We compare three teacher types: standard (no defense), ADS (the current state-of-the-art by~\cite{savani2025antidistillation}), and our PoE teacher. We record the reasoning trace for each prompt by each teacher. For ADS and PoE, we sweep the defense strength to trace out a utility--distillability frontier, with teacher accuracy measuring utility and downstream student accuracy measuring leakage. Table~\ref{tab:seed_aggregates} reports representative points, while Figure~\ref{fig:teacher_student_tradeoff} shows the full frontier.
The aforementioned representative points are chosen to compare defenses at similar and reasonable teacher accuracies.

\textbf{Passive and adaptive students.}
For each teacher, we train two students using the same distillation pipeline and differing only in how released traces are weighted. The passive student trains on released traces as given. The adaptive student applies our best-response evaluation rule, exponentially reweighting traces with higher estimated value before training. 
We deliberately restrict the attacker to reweighting traces that have already been released: it performs no prompt selection, additional querying, or trace rewriting. Thus, the leakage we measure lower-bounds what a more capable adaptive attacker could recover.

\textbf{Training and metrics.}
We fine-tune the student with LoRA~\citep{hu2022lora} for 3 epochs using rank 128, $\alpha=128$, learning rate $5\times 10^{-4}$, cosine decay, batch size 6, gradient accumulation 4, and maximum sequence length 2{,}048. Our main metrics are teacher accuracy and student accuracy after distillation. We also report the relative gain of the adaptive student over the passive student. Unless otherwise noted, all results are mean $\pm$ standard error over three seeds.

\vspace{-1mm}
\subsection{Results}
\vspace{-1mm}

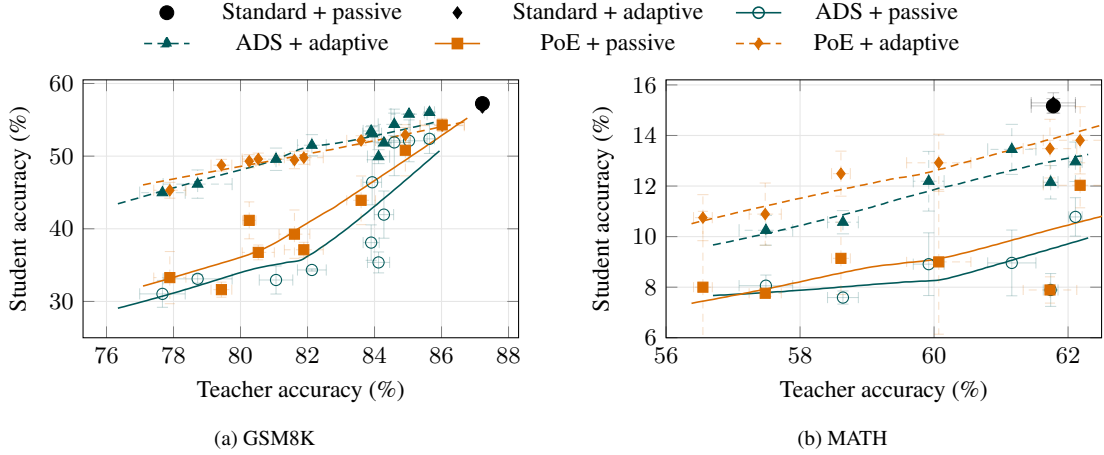
\begin{figure*}[t!]
\vspace{-5mm}
\centering
\begin{tikzpicture}
\begin{axis}[
    hide axis,
    xmin=0, xmax=1, ymin=0, ymax=1,
    legend columns=3,
    legend style={draw=none, fill=none, font=\small, column sep=10pt},
]
\addlegendimage{only marks, mark=*, mark size=2.6pt, color=black}
\addlegendentry{Standard + passive}
\addlegendimage{only marks, mark=diamond*, mark size=2.2pt, color=black}
\addlegendentry{Standard + adaptive}
\addlegendimage{semithick, solid, mark=o, mark size=2.1pt, color=teal!70!black}
\addlegendentry{ADS + passive}
\addlegendimage{semithick, densely dashed, mark=triangle*, mark size=2.3pt, color=teal!70!black}
\addlegendentry{ADS + adaptive}
\addlegendimage{semithick, solid, mark=square*, mark size=1.9pt, color=orange!85!black}
\addlegendentry{PoE + passive}
\addlegendimage{semithick, densely dashed, mark=diamond*, mark size=2.1pt, color=orange!85!black}
\addlegendentry{PoE + adaptive}
\end{axis}
\end{tikzpicture}
\vspace{2pt}
\begin{subfigure}[t]{0.49\textwidth}
\centering
\begin{tikzpicture}
\begin{axis}[
    width=\linewidth,
    height=0.68\linewidth,
    xlabel={Teacher accuracy (\%)},
    ylabel={Student accuracy (\%)},
    xmin=75.3, xmax=88.3,
    ymin=25.0, ymax=60.5,
    grid=both,
    major grid style={draw=gray!20},
    minor grid style={draw=gray!10},
    tick label style={font=\small},
    label style={font=\small},
]

\addplot+[
    only marks,
    mark=*,
    mark size=2.6pt,
    color=black,
    mark options={solid},
    error bars/.cd,
    x dir=both, x explicit,
    y dir=both, y explicit,
    error bar style={opacity=0.22},
] coordinates {
    (87.22,57.24) +- (0.04,0.25)
};

\addplot+[
    only marks,
    mark=diamond*,
    mark size=2.2pt,
    color=black,
    mark options={solid},
    error bars/.cd,
    x dir=both, x explicit,
    y dir=both, y explicit,
    error bar style={opacity=0.22},
] coordinates {
    (87.22,56.74) +- (0.04,0.17)
};

\addplot+[
    semithick,
    smooth,
    solid,
    color=teal!70!black,
    mark=none,
] coordinates {
    (76.34,29.07)
    (76.44,29.19)
    (76.54,29.30)
    (76.63,29.42)
    (76.73,29.54)
    (76.83,29.66)
    (76.93,29.78)
    (77.02,29.90)
    (77.12,30.02)
    (77.22,30.14)
    (77.31,30.27)
    (77.41,30.39)
    (77.51,30.51)
    (77.60,30.64)
    (77.70,30.76)
    (77.80,30.89)
    (77.89,31.02)
    (77.99,31.14)
    (78.09,31.27)
    (78.19,31.40)
    (78.28,31.54)
    (78.38,31.67)
    (78.48,31.80)
    (78.57,31.94)
    (78.67,32.07)
    (78.77,32.21)
    (78.86,32.35)
    (78.96,32.48)
    (79.06,32.62)
    (79.15,32.76)
    (79.25,32.90)
    (79.35,33.04)
    (79.45,33.18)
    (79.54,33.32)
    (79.64,33.46)
    (79.74,33.60)
    (79.83,33.74)
    (79.93,33.88)
    (80.03,34.01)
    (80.12,34.14)
    (80.22,34.26)
    (80.32,34.38)
    (80.41,34.50)
    (80.51,34.60)
    (80.61,34.70)
    (80.71,34.79)
    (80.80,34.88)
    (80.90,34.96)
    (81.00,35.03)
    (81.09,35.11)
    (81.19,35.18)
    (81.29,35.28)
    (81.38,35.38)
    (81.48,35.43)
    (81.58,35.48)
    (81.67,35.54)
    (81.77,35.64)
    (81.87,35.83)
    (81.97,36.14)
    (82.06,36.39)
    (82.16,36.64)
    (82.26,36.94)
    (82.35,37.25)
    (82.45,37.56)
    (82.55,37.86)
    (82.64,38.18)
    (82.74,38.50)
    (82.84,38.82)
    (82.93,39.16)
    (83.03,39.49)
    (83.13,39.84)
    (83.23,40.19)
    (83.32,40.54)
    (83.42,40.90)
    (83.52,41.26)
    (83.61,41.62)
    (83.71,41.98)
    (83.81,42.35)
    (83.90,42.72)
    (84.00,43.09)
    (84.10,43.46)
    (84.19,43.84)
    (84.29,44.21)
    (84.39,44.59)
    (84.49,44.97)
    (84.58,45.34)
    (84.68,45.72)
    (84.78,46.10)
    (84.87,46.48)
    (84.97,46.86)
    (85.07,47.24)
    (85.16,47.63)
    (85.26,48.01)
    (85.36,48.39)
    (85.45,48.78)
    (85.55,49.17)
    (85.65,49.55)
    (85.74,49.94)
    (85.84,50.33)
    (85.94,50.72)
};

\addplot+[
    only marks,
    mark=o,
    mark size=2.1pt,
    color=teal!70!black,
    mark options={solid},
    error bars/.cd,
    x dir=both, x explicit,
    y dir=both, y explicit,
    error bar style={opacity=0.22},
] coordinates {
    (85.64,52.38) +- (0.15,2.01)
    (85.03,52.11) +- (0.54,2.88)
    (84.59,51.89) +- (0.48,4.58)
    (83.93,46.40) +- (0.28,3.10)
    (84.28,41.95) +- (0.29,3.25)
    (83.90,38.10) +- (0.24,2.44)
    (84.12,35.37) +- (0.33,1.44)
    (82.13,34.33) +- (0.43,0.17)
    (81.06,32.95) +- (0.50,1.93)
    (78.72,33.11) +- (1.03,0.59)
    (77.67,31.04) +- (0.68,1.85)
};

\addplot+[
    semithick,
    smooth,
    densely dashed,
    color=teal!70!black,
    mark=none,
] coordinates {
    (76.34,43.44)
    (76.44,43.57)
    (76.54,43.70)
    (76.63,43.83)
    (76.73,43.96)
    (76.83,44.08)
    (76.93,44.21)
    (77.02,44.34)
    (77.12,44.47)
    (77.22,44.60)
    (77.31,44.72)
    (77.41,44.85)
    (77.51,44.98)
    (77.60,45.11)
    (77.70,45.24)
    (77.80,45.36)
    (77.89,45.49)
    (77.99,45.62)
    (78.09,45.74)
    (78.19,45.87)
    (78.28,45.99)
    (78.38,46.12)
    (78.48,46.24)
    (78.57,46.36)
    (78.67,46.49)
    (78.77,46.61)
    (78.86,46.73)
    (78.96,46.85)
    (79.06,46.97)
    (79.15,47.09)
    (79.25,47.21)
    (79.35,47.32)
    (79.45,47.44)
    (79.54,47.56)
    (79.64,47.67)
    (79.74,47.79)
    (79.83,47.90)
    (79.93,48.02)
    (80.03,48.14)
    (80.12,48.25)
    (80.22,48.37)
    (80.32,48.49)
    (80.41,48.61)
    (80.51,48.74)
    (80.61,48.87)
    (80.71,49.01)
    (80.80,49.16)
    (80.90,49.31)
    (81.00,49.48)
    (81.09,49.65)
    (81.19,49.83)
    (81.29,49.99)
    (81.38,50.17)
    (81.48,50.37)
    (81.58,50.57)
    (81.67,50.76)
    (81.77,50.95)
    (81.87,51.08)
    (81.97,51.16)
    (82.06,51.25)
    (82.16,51.36)
    (82.26,51.43)
    (82.35,51.48)
    (82.45,51.53)
    (82.55,51.58)
    (82.64,51.64)
    (82.74,51.70)
    (82.84,51.76)
    (82.93,51.83)
    (83.03,51.91)
    (83.13,51.98)
    (83.23,52.07)
    (83.32,52.15)
    (83.42,52.24)
    (83.52,52.33)
    (83.61,52.42)
    (83.71,52.51)
    (83.81,52.60)
    (83.90,52.70)
    (84.00,52.79)
    (84.10,52.89)
    (84.19,52.99)
    (84.29,53.08)
    (84.39,53.18)
    (84.49,53.28)
    (84.58,53.38)
    (84.68,53.48)
    (84.78,53.58)
    (84.87,53.68)
    (84.97,53.78)
    (85.07,53.88)
    (85.16,53.98)
    (85.26,54.08)
    (85.36,54.19)
    (85.45,54.29)
    (85.55,54.40)
    (85.65,54.50)
    (85.74,54.61)
    (85.84,54.72)
    (85.94,54.83)
};

\addplot+[
    only marks,
    mark=triangle*,
    mark size=2.3pt,
    color=teal!70!black,
    mark options={solid},
    error bars/.cd,
    x dir=both, x explicit,
    y dir=both, y explicit,
    error bar style={opacity=0.22},
] coordinates {
    (85.64,56.02) +- (0.15,0.06)
    (85.03,55.78) +- (0.54,0.06)
    (84.59,54.37) +- (0.48,1.55)
    (83.93,53.07) +- (0.28,1.12)
    (84.28,51.82) +- (0.29,0.46)
    (83.90,53.46) +- (0.24,0.59)
    (84.12,49.93) +- (0.33,0.96)
    (82.13,51.50) +- (0.43,1.46)
    (81.06,49.57) +- (0.50,1.54)
    (78.72,46.15) +- (1.03,1.95)
    (77.67,44.97) +- (0.68,0.17)
};

\addplot+[
    semithick,
    smooth,
    solid,
    color=orange!85!black,
    mark=none,
] coordinates {
    (77.09,32.12)
    (77.19,32.24)
    (77.28,32.36)
    (77.38,32.49)
    (77.48,32.62)
    (77.58,32.75)
    (77.68,32.88)
    (77.77,33.01)
    (77.87,33.14)
    (77.97,33.27)
    (78.07,33.41)
    (78.16,33.54)
    (78.26,33.67)
    (78.36,33.81)
    (78.46,33.94)
    (78.55,34.08)
    (78.65,34.21)
    (78.75,34.35)
    (78.85,34.48)
    (78.95,34.61)
    (79.04,34.75)
    (79.14,34.88)
    (79.24,35.01)
    (79.34,35.14)
    (79.43,35.28)
    (79.53,35.41)
    (79.63,35.54)
    (79.73,35.68)
    (79.83,35.82)
    (79.92,35.96)
    (80.02,36.10)
    (80.12,36.25)
    (80.22,36.40)
    (80.31,36.56)
    (80.41,36.73)
    (80.51,36.90)
    (80.61,37.09)
    (80.71,37.29)
    (80.80,37.50)
    (80.90,37.73)
    (81.00,37.97)
    (81.10,38.22)
    (81.19,38.47)
    (81.29,38.74)
    (81.39,39.02)
    (81.49,39.31)
    (81.59,39.60)
    (81.68,39.89)
    (81.78,40.18)
    (81.88,40.46)
    (81.98,40.73)
    (82.07,41.00)
    (82.17,41.26)
    (82.27,41.54)
    (82.37,41.81)
    (82.46,42.07)
    (82.56,42.33)
    (82.66,42.58)
    (82.76,42.83)
    (82.86,43.13)
    (82.95,43.44)
    (83.05,43.74)
    (83.15,44.04)
    (83.25,44.34)
    (83.34,44.64)
    (83.44,44.93)
    (83.54,45.23)
    (83.64,45.53)
    (83.74,45.83)
    (83.83,46.12)
    (83.93,46.42)
    (84.03,46.72)
    (84.13,47.01)
    (84.22,47.31)
    (84.32,47.61)
    (84.42,47.91)
    (84.52,48.21)
    (84.62,48.51)
    (84.71,48.81)
    (84.81,49.11)
    (84.91,49.41)
    (85.01,49.72)
    (85.10,50.02)
    (85.20,50.32)
    (85.30,50.63)
    (85.40,50.93)
    (85.50,51.24)
    (85.59,51.55)
    (85.69,51.85)
    (85.79,52.16)
    (85.89,52.46)
    (85.98,52.77)
    (86.08,53.08)
    (86.18,53.39)
    (86.28,53.69)
    (86.37,54.00)
    (86.47,54.31)
    (86.57,54.62)
    (86.67,54.93)
    (86.77,55.23)
};

\addplot+[
    only marks,
    mark=square*,
    mark size=1.9pt,
    color=orange!85!black,
    mark options={solid},
    error bars/.cd,
    x dir=both, x explicit,
    y dir=both, y explicit,
    error bar style={opacity=0.22},
] coordinates {
    (86.02,54.29) +- (0.66,0.90)
    (84.92,50.81) +- (0.15,3.52)
    (83.60,43.92) +- (0.43,3.35)
    (81.89,37.11) +- (0.58,1.03)
    (81.61,39.26) +- (0.46,3.33)
    (80.53,36.75) +- (0.48,1.02)
    (80.26,41.16) +- (0.20,2.54)
    (79.43,31.62) +- (0.31,1.04)
    (77.89,33.28) +- (0.48,3.59)
};

\addplot+[
    semithick,
    smooth,
    densely dashed,
    color=orange!85!black,
    mark=none,
] coordinates {
    (77.09,46.04)
    (77.19,46.13)
    (77.28,46.22)
    (77.38,46.31)
    (77.48,46.40)
    (77.58,46.48)
    (77.68,46.57)
    (77.77,46.66)
    (77.87,46.75)
    (77.97,46.83)
    (78.07,46.92)
    (78.16,47.01)
    (78.26,47.09)
    (78.36,47.18)
    (78.46,47.26)
    (78.55,47.35)
    (78.65,47.43)
    (78.75,47.52)
    (78.85,47.60)
    (78.95,47.68)
    (79.04,47.77)
    (79.14,47.85)
    (79.24,47.94)
    (79.34,48.02)
    (79.43,48.10)
    (79.53,48.19)
    (79.63,48.27)
    (79.73,48.35)
    (79.83,48.44)
    (79.92,48.52)
    (80.02,48.60)
    (80.12,48.69)
    (80.22,48.77)
    (80.31,48.85)
    (80.41,48.94)
    (80.51,49.02)
    (80.61,49.10)
    (80.71,49.18)
    (80.80,49.26)
    (80.90,49.33)
    (81.00,49.41)
    (81.10,49.48)
    (81.19,49.56)
    (81.29,49.65)
    (81.39,49.75)
    (81.49,49.84)
    (81.59,49.94)
    (81.68,50.03)
    (81.78,50.13)
    (81.88,50.21)
    (81.98,50.30)
    (82.07,50.38)
    (82.17,50.46)
    (82.27,50.54)
    (82.37,50.63)
    (82.46,50.71)
    (82.56,50.79)
    (82.66,50.86)
    (82.76,50.94)
    (82.86,51.04)
    (82.95,51.14)
    (83.05,51.23)
    (83.15,51.33)
    (83.25,51.42)
    (83.34,51.51)
    (83.44,51.60)
    (83.54,51.69)
    (83.64,51.78)
    (83.74,51.87)
    (83.83,51.97)
    (83.93,52.06)
    (84.03,52.15)
    (84.13,52.24)
    (84.22,52.33)
    (84.32,52.42)
    (84.42,52.52)
    (84.52,52.61)
    (84.62,52.70)
    (84.71,52.80)
    (84.81,52.89)
    (84.91,52.99)
    (85.01,53.08)
    (85.10,53.18)
    (85.20,53.27)
    (85.30,53.36)
    (85.40,53.46)
    (85.50,53.55)
    (85.59,53.65)
    (85.69,53.74)
    (85.79,53.84)
    (85.89,53.94)
    (85.98,54.03)
    (86.08,54.13)
    (86.18,54.22)
    (86.28,54.31)
    (86.37,54.41)
    (86.47,54.50)
    (86.57,54.60)
    (86.67,54.69)
    (86.77,54.79)
};

\addplot+[
    only marks,
    mark=diamond*,
    mark size=2.1pt,
    color=orange!85!black,
    mark options={solid},
    error bars/.cd,
    x dir=both, x explicit,
    y dir=both, y explicit,
    error bar style={opacity=0.22},
] coordinates {
    (86.02,54.20) +- (0.66,0.17)
    (84.92,52.91) +- (0.15,1.03)
    (83.60,52.16) +- (0.43,0.65)
    (81.89,49.79) +- (0.58,0.91)
    (81.61,49.46) +- (0.46,1.19)
    (80.53,49.60) +- (0.48,0.82)
    (80.26,49.30) +- (0.20,0.91)
    (79.43,48.75) +- (0.31,0.81)
    (77.89,45.30) +- (0.48,1.09)
};

\end{axis}
\end{tikzpicture}
\caption{GSM8K}
\label{fig:gsm8k_teacher_student}
\end{subfigure}
\hfill
\begin{subfigure}[t]{0.49\textwidth}
\centering
\begin{tikzpicture}
\begin{axis}[
    width=\linewidth,
    height=0.68\linewidth,
    xlabel={Teacher accuracy (\%)},
    ylabel={Student accuracy (\%)},
    xmin=56, xmax=62.5,
    ymin=6, ymax=16.2,
    grid=both,
    major grid style={draw=gray!20},
    minor grid style={draw=gray!10},
    tick label style={font=\small},
    label style={font=\small},
]

\addplot+[
    only marks,
    mark=*,
    mark size=2.6pt,
    color=black,
    mark options={solid},
    error bars/.cd,
    x dir=both, x explicit,
    y dir=both, y explicit,
    error bar style={opacity=0.22},
] coordinates {
    (61.78,15.17) +- (0.33,0.29)
};

\addplot+[
    only marks,
    mark=diamond*,
    mark size=2.2pt,
    color=black,
    mark options={solid},
    error bars/.cd,
    x dir=both, x explicit,
    y dir=both, y explicit,
    error bar style={opacity=0.22},
] coordinates {
    (61.78,15.29) +- (0.33,0.40)
};

\addplot+[
    semithick,
    smooth,
    solid,
    color=teal!70!black,
    mark=none,
] coordinates {
    (56.70,7.67)
    (56.76,7.68)
    (56.81,7.69)
    (56.87,7.69)
    (56.93,7.70)
    (56.98,7.71)
    (57.04,7.72)
    (57.10,7.73)
    (57.15,7.73)
    (57.21,7.74)
    (57.27,7.75)
    (57.32,7.76)
    (57.38,7.77)
    (57.44,7.78)
    (57.49,7.79)
    (57.55,7.80)
    (57.61,7.81)
    (57.66,7.82)
    (57.72,7.82)
    (57.77,7.83)
    (57.83,7.85)
    (57.89,7.86)
    (57.94,7.87)
    (58.00,7.88)
    (58.06,7.89)
    (58.11,7.90)
    (58.17,7.91)
    (58.23,7.92)
    (58.28,7.93)
    (58.34,7.95)
    (58.40,7.96)
    (58.45,7.97)
    (58.51,7.98)
    (58.57,8.00)
    (58.62,8.01)
    (58.68,8.02)
    (58.74,8.03)
    (58.79,8.04)
    (58.85,8.06)
    (58.91,8.07)
    (58.96,8.08)
    (59.02,8.09)
    (59.08,8.11)
    (59.13,8.12)
    (59.19,8.13)
    (59.25,8.14)
    (59.30,8.15)
    (59.36,8.16)
    (59.42,8.17)
    (59.47,8.18)
    (59.53,8.20)
    (59.58,8.21)
    (59.64,8.21)
    (59.70,8.22)
    (59.75,8.23)
    (59.81,8.24)
    (59.87,8.24)
    (59.92,8.25)
    (59.98,8.26)
    (60.04,8.27)
    (60.09,8.29)
    (60.15,8.32)
    (60.21,8.35)
    (60.26,8.39)
    (60.32,8.43)
    (60.38,8.47)
    (60.43,8.51)
    (60.49,8.55)
    (60.55,8.59)
    (60.60,8.63)
    (60.66,8.68)
    (60.72,8.72)
    (60.77,8.77)
    (60.83,8.81)
    (60.89,8.85)
    (60.94,8.90)
    (61.00,8.94)
    (61.06,8.98)
    (61.11,9.03)
    (61.17,9.07)
    (61.23,9.12)
    (61.28,9.16)
    (61.34,9.20)
    (61.39,9.25)
    (61.45,9.29)
    (61.51,9.33)
    (61.56,9.38)
    (61.62,9.42)
    (61.68,9.46)
    (61.73,9.51)
    (61.79,9.55)
    (61.85,9.59)
    (61.90,9.64)
    (61.96,9.68)
    (62.02,9.73)
    (62.07,9.77)
    (62.13,9.81)
    (62.19,9.86)
    (62.24,9.90)
    (62.30,9.95)
};

\addplot+[
    only marks,
    mark=o,
    mark size=2.1pt,
    color=teal!70!black,
    mark options={solid},
    error bars/.cd,
    x dir=both, x explicit,
    y dir=both, y explicit,
    error bar style={opacity=0.22},
] coordinates {
    (62.11,10.78) +- (0.10,0.76)
    (61.74,7.89) +- (0.11,0.65)
    (61.16,8.96) +- (0.36,1.30)
    (59.92,8.91) +- (0.23,1.24)
    (58.64,7.59) +- (0.23,0.26)
    (57.49,8.06) +- (0.40,0.42)
};

\addplot+[
    semithick,
    smooth,
    densely dashed,
    color=teal!70!black,
    mark=none,
] coordinates {
    (56.70,9.67)
    (56.76,9.70)
    (56.81,9.73)
    (56.87,9.76)
    (56.93,9.79)
    (56.98,9.82)
    (57.04,9.86)
    (57.10,9.89)
    (57.15,9.92)
    (57.21,9.95)
    (57.27,9.99)
    (57.32,10.02)
    (57.38,10.05)
    (57.44,10.09)
    (57.49,10.12)
    (57.55,10.15)
    (57.61,10.19)
    (57.66,10.22)
    (57.72,10.26)
    (57.77,10.29)
    (57.83,10.33)
    (57.89,10.37)
    (57.94,10.40)
    (58.00,10.44)
    (58.06,10.47)
    (58.11,10.51)
    (58.17,10.55)
    (58.23,10.59)
    (58.28,10.62)
    (58.34,10.66)
    (58.40,10.70)
    (58.45,10.74)
    (58.51,10.78)
    (58.57,10.82)
    (58.62,10.85)
    (58.68,10.89)
    (58.74,10.93)
    (58.79,10.97)
    (58.85,11.01)
    (58.91,11.05)
    (58.96,11.09)
    (59.02,11.13)
    (59.08,11.17)
    (59.13,11.22)
    (59.19,11.26)
    (59.25,11.30)
    (59.30,11.35)
    (59.36,11.39)
    (59.42,11.44)
    (59.47,11.48)
    (59.53,11.52)
    (59.58,11.56)
    (59.64,11.60)
    (59.70,11.64)
    (59.75,11.68)
    (59.81,11.72)
    (59.87,11.76)
    (59.92,11.80)
    (59.98,11.84)
    (60.04,11.89)
    (60.09,11.93)
    (60.15,11.96)
    (60.21,11.99)
    (60.26,12.02)
    (60.32,12.06)
    (60.38,12.09)
    (60.43,12.13)
    (60.49,12.17)
    (60.55,12.21)
    (60.60,12.25)
    (60.66,12.29)
    (60.72,12.32)
    (60.77,12.36)
    (60.83,12.40)
    (60.89,12.44)
    (60.94,12.48)
    (61.00,12.52)
    (61.06,12.55)
    (61.11,12.59)
    (61.17,12.63)
    (61.23,12.66)
    (61.28,12.70)
    (61.34,12.73)
    (61.39,12.77)
    (61.45,12.80)
    (61.51,12.83)
    (61.56,12.87)
    (61.62,12.90)
    (61.68,12.93)
    (61.73,12.96)
    (61.79,12.99)
    (61.85,13.02)
    (61.90,13.05)
    (61.96,13.08)
    (62.02,13.11)
    (62.07,13.14)
    (62.13,13.17)
    (62.19,13.19)
    (62.24,13.22)
    (62.30,13.25)
};

\addplot+[
    only marks,
    mark=triangle*,
    mark size=2.3pt,
    color=teal!70!black,
    mark options={solid},
    error bars/.cd,
    x dir=both, x explicit,
    y dir=both, y explicit,
    error bar style={opacity=0.22},
] coordinates {
    (62.11,12.96) +- (0.10,0.78)
    (61.74,12.15) +- (0.11,0.66)
    (61.16,13.45) +- (0.36,0.99)
    (59.92,12.19) +- (0.23,1.18)
    (58.64,10.57) +- (0.23,0.46)
    (57.49,10.25) +- (0.40,0.59)
};

\addplot+[
    semithick,
    smooth,
    solid,
    color=orange!85!black,
    mark=none,
] coordinates {
    (56.38,7.36)
    (56.45,7.40)
    (56.51,7.43)
    (56.58,7.46)
    (56.64,7.49)
    (56.71,7.52)
    (56.78,7.56)
    (56.84,7.59)
    (56.91,7.62)
    (56.97,7.66)
    (57.04,7.69)
    (57.11,7.72)
    (57.17,7.76)
    (57.24,7.79)
    (57.30,7.83)
    (57.37,7.86)
    (57.44,7.90)
    (57.50,7.93)
    (57.57,7.97)
    (57.64,8.01)
    (57.70,8.04)
    (57.77,8.08)
    (57.83,8.12)
    (57.90,8.16)
    (57.97,8.19)
    (58.03,8.23)
    (58.10,8.27)
    (58.16,8.31)
    (58.23,8.35)
    (58.30,8.38)
    (58.36,8.42)
    (58.43,8.46)
    (58.49,8.50)
    (58.56,8.53)
    (58.63,8.57)
    (58.69,8.61)
    (58.76,8.64)
    (58.82,8.68)
    (58.89,8.71)
    (58.96,8.75)
    (59.02,8.78)
    (59.09,8.81)
    (59.15,8.83)
    (59.22,8.86)
    (59.29,8.88)
    (59.35,8.89)
    (59.42,8.91)
    (59.48,8.93)
    (59.55,8.96)
    (59.62,8.98)
    (59.68,8.99)
    (59.75,9.01)
    (59.82,9.03)
    (59.88,9.05)
    (59.95,9.07)
    (60.01,9.10)
    (60.08,9.13)
    (60.15,9.17)
    (60.21,9.20)
    (60.28,9.24)
    (60.34,9.28)
    (60.41,9.33)
    (60.48,9.37)
    (60.54,9.41)
    (60.61,9.46)
    (60.67,9.51)
    (60.74,9.55)
    (60.81,9.60)
    (60.87,9.65)
    (60.94,9.69)
    (61.00,9.74)
    (61.07,9.79)
    (61.14,9.84)
    (61.20,9.89)
    (61.27,9.93)
    (61.33,9.98)
    (61.40,10.03)
    (61.47,10.08)
    (61.53,10.13)
    (61.60,10.17)
    (61.66,10.22)
    (61.73,10.27)
    (61.80,10.32)
    (61.86,10.36)
    (61.93,10.41)
    (62.00,10.45)
    (62.06,10.50)
    (62.13,10.55)
    (62.19,10.59)
    (62.26,10.64)
    (62.33,10.68)
    (62.39,10.73)
    (62.46,10.77)
    (62.52,10.82)
    (62.59,10.86)
    (62.66,10.91)
    (62.72,10.95)
    (62.79,11.00)
    (62.85,11.04)
    (62.92,11.09)
};

\addplot+[
    only marks,
    mark=square*,
    mark size=1.9pt,
    color=orange!85!black,
    mark options={solid},
    error bars/.cd,
    x dir=both, x explicit,
    y dir=both, y explicit,
    error bar style={opacity=0.22},
] coordinates {
    (62.18,12.03) +- (0.55,0.89)
    (61.73,7.89) +- (0.40,0.52)
    (60.07,9.00) +- (0.48,2.86)
    (58.61,9.14) +- (0.14,1.43)
    (57.48,7.76) +- (0.24,0.12)
    (56.55,8.00) +- (0.14,3.00)
};

\addplot+[
    semithick,
    smooth,
    densely dashed,
    color=orange!85!black,
    mark=none,
] coordinates {
    (56.38,10.51)
    (56.45,10.56)
    (56.51,10.60)
    (56.58,10.64)
    (56.64,10.68)
    (56.71,10.73)
    (56.78,10.77)
    (56.84,10.81)
    (56.91,10.85)
    (56.97,10.89)
    (57.04,10.94)
    (57.11,10.98)
    (57.17,11.02)
    (57.24,11.06)
    (57.30,11.10)
    (57.37,11.14)
    (57.44,11.18)
    (57.50,11.22)
    (57.57,11.26)
    (57.64,11.30)
    (57.70,11.34)
    (57.77,11.38)
    (57.83,11.42)
    (57.90,11.46)
    (57.97,11.50)
    (58.03,11.53)
    (58.10,11.57)
    (58.16,11.61)
    (58.23,11.65)
    (58.30,11.68)
    (58.36,11.72)
    (58.43,11.76)
    (58.49,11.80)
    (58.56,11.83)
    (58.63,11.87)
    (58.69,11.91)
    (58.76,11.94)
    (58.82,11.98)
    (58.89,12.02)
    (58.96,12.05)
    (59.02,12.09)
    (59.09,12.12)
    (59.15,12.16)
    (59.22,12.20)
    (59.29,12.23)
    (59.35,12.27)
    (59.42,12.30)
    (59.48,12.33)
    (59.55,12.35)
    (59.62,12.38)
    (59.68,12.41)
    (59.75,12.44)
    (59.82,12.48)
    (59.88,12.52)
    (59.95,12.56)
    (60.01,12.61)
    (60.08,12.66)
    (60.15,12.70)
    (60.21,12.75)
    (60.28,12.80)
    (60.34,12.85)
    (60.41,12.90)
    (60.48,12.95)
    (60.54,12.99)
    (60.61,13.04)
    (60.67,13.09)
    (60.74,13.14)
    (60.81,13.18)
    (60.87,13.23)
    (60.94,13.28)
    (61.00,13.33)
    (61.07,13.37)
    (61.14,13.42)
    (61.20,13.47)
    (61.27,13.51)
    (61.33,13.56)
    (61.40,13.61)
    (61.47,13.65)
    (61.53,13.70)
    (61.60,13.75)
    (61.66,13.80)
    (61.73,13.84)
    (61.80,13.89)
    (61.86,13.94)
    (61.93,13.99)
    (62.00,14.03)
    (62.06,14.08)
    (62.13,14.13)
    (62.19,14.18)
    (62.26,14.23)
    (62.33,14.28)
    (62.39,14.33)
    (62.46,14.38)
    (62.52,14.43)
    (62.59,14.48)
    (62.66,14.53)
    (62.72,14.58)
    (62.79,14.63)
    (62.85,14.68)
    (62.92,14.73)
};

\addplot+[
    only marks,
    mark=diamond*,
    mark size=2.1pt,
    color=orange!85!black,
    mark options={solid},
    error bars/.cd,
    x dir=both, x explicit,
    y dir=both, y explicit,
    error bar style={opacity=0.22},
] coordinates {
    (62.18,13.81) +- (0.55,1.33)
    (61.73,13.48) +- (0.40,1.16)
    (60.07,12.92) +- (0.48,1.13)
    (58.61,12.49) +- (0.14,0.89)
    (57.48,10.89) +- (0.24,1.23)
    (56.55,10.75) +- (0.14,0.91)
};

\end{axis}
\end{tikzpicture}
\caption{MATH}
\label{fig:math_teacher_student}
\end{subfigure}
\caption{Utility--distillability frontiers under passive and adaptive evaluation. Adaptive evaluation shifts the frontier upward, revealing substantially more leakage than passive evaluation. Curves are smooth fits to seed-level points, and error bars show standard error of the mean across 3 seeds.}
\label{fig:teacher_student_tradeoff}
\vspace{-1em}
\end{figure*}

\textbf{Adaptive evaluation reveals substantially more leakage than passive evaluation.}
Across datasets, defenses appear much more robust under passive evaluation than under adaptive evaluation. On \texttt{GSM8K}, ADS reduces student accuracy to $34\%$ under passive evaluation, but the adaptive student reaches $52\%$ against the same teacher; PoE shows the same pattern, rising from $39\%$ to $49\%$. \texttt{MATH} is noisier but directionally identical: both ADS and PoE rise from $9\%$ to $13\%$, ignoring decimals. These gaps are large enough to change the qualitative interpretation of the defenses: methods that look strong against a passive student can leak much more under an adaptive one.

\textbf{Under adaptive evaluation, the gap between ADS and PoE narrows substantially.}
At comparable teacher accuracy, ADS usually looks stronger than PoE under passive evaluation. On \texttt{GSM8K}, for example, ADS reaches lower passive student accuracy than PoE ($34\%$ vs. $39\%$), but the difference is much smaller under adaptive evaluation ($52\%$ vs. $49\%$). 
The practical comparison changes further once runtime is included: on \texttt{GSM8K}, PoE increases teacher generation time by 64\% ($1.64\times$), whereas ADS increases it by $193\%$ ($2.93\times$). 
Overall, ADS appears stronger under passive evaluation at comparable teacher accuracy, but this advantage largely disappears under adaptive evaluation. At the representative \texttt{GSM8K} operating point in Table~\ref{tab:seed_aggregates}, PoE even yields slightly lower adaptive student accuracy than ADS while being substantially cheaper. Thus, the main conclusion is not that either method uniformly dominates, but that adaptive evaluation and runtime cost substantially change the practical comparison, with the latter favoring PoE.

\textbf{PoE preserves higher-quality reasoning traces.}
Student accuracy does not capture whether the teacher's reasoning traces remain useful to a user. To measure this, we score trace quality with a rubric-based Claude Sonnet 4.6~\citep{anthropic2026claude-sonnet-4-6-system-card} judge that evaluates whether reasoning steps are identifiable, relevant, and checkable. The rubric scores traces from $1$ (no auditable reasoning) to $5$ (every step is identifiable, relevant, and checkable), and explicitly instructs the judge to ignore length and style so that verbose and concise traces are treated equally; see Appendix~\ref{sec:auditability-prompt} for the full prompt.
Figure~\ref{fig:trace-quality} shows that PoE produces more high-scoring traces than ADS on both datasets, while remaining much closer to the standard teacher. Qualitative examples in Appendix~\ref{app:trace_examples} show the same pattern. This suggests that PoE better preserves auditable reasoning while still reducing distillation leakage.
This is consistent with the design of PoE: it suppresses outputs with large teacher-student likelihood gaps while retaining those that are likely under the teacher, whereas ADS directly optimizes against the student's learning signal and may more strongly distort readable reasoning.
To sanity-check the rubric-based judge, we collected human ratings on $30$ traces sampled evenly across teacher types. The human rater was blind to which teacher produced each trace and used the same $1–5$ auditability rubric as the judge. Quadratic-weighted Cohen’s $\kappa$ between human and judge ratings was $0.76$, and the judge's mean absolute error relative to human ratings was $0.40$ rubric points, with mean signed error $+0.13$. This suggests that the judge is reasonably well calibrated for the auditability rubric, with no large systematic offset.

\textbf{Defenses shift the frontier, but much less under adaptive attack.}
Both ADS and PoE improve the utility--distillability trade-off relative to the standard teacher, but the gain is much smaller against adaptive than against  passive evaluation. 
On \texttt{GSM8K}, ADS sacrifices about $5$ teacher-accuracy points and reduces student accuracy by $23$ points under passive evaluation, but by only $5$ points under adaptive evaluation. PoE shows the same flattening: for about a $6$-point teacher-accuracy loss, it reduces student accuracy by $18$ points under passive evaluation but only $7$ points under adaptive evaluation.
The threat-model implication is as follows: if the adversary is assumed to be passive, ADS looks  stronger, but under an adaptive adversary the gap between ADS and PoE narrows substantially, making PoE especially attractive once runtime and trace quality are taken into account.

\textbf{Ablations.} Additional experiments in Appendix~\ref{app:ablations} show that the passive--adaptive gap is not explained by trivial filtering or a particular defense proxy. Straightforward filtering improves passive SFT but remains substantially below adaptive reweighting, while replacing the 3B defense proxy with the 1.5B model from the same family changes defense calibration but preserves an adaptive-–passive gap at usable operating points. We further evaluate the effect on an additional defense family, providing evidence that the phenomenon is not specific to ADS.

\begin{figure}
    \vspace{-1cm}
    \centering
    \begin{subfigure}[b]{0.48\linewidth}
        \centering
        \includegraphics[width=0.95\linewidth]{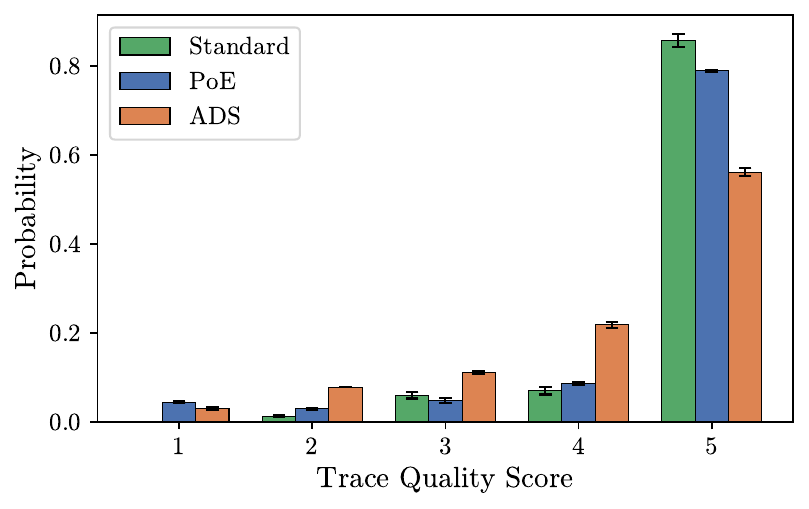}
        \caption{GSM8K}
    \end{subfigure}
    \hfill
    \begin{subfigure}[b]{0.48\linewidth}
        \centering
        \includegraphics[width=0.95\linewidth]{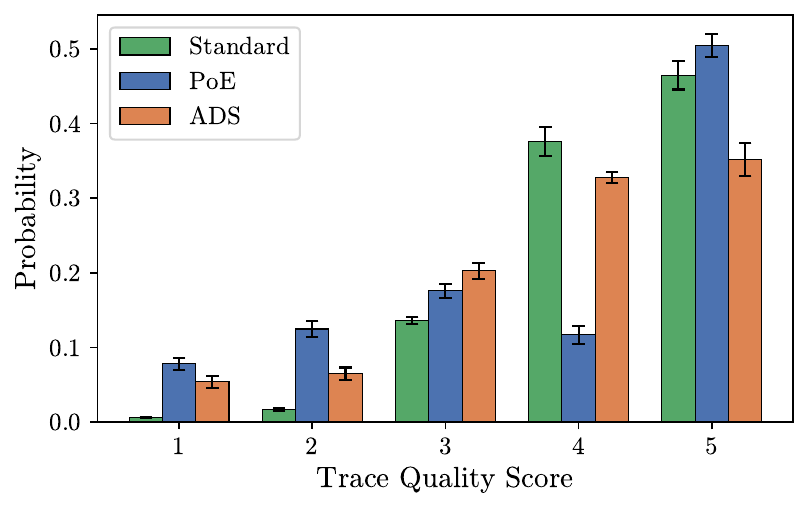}
        \caption{MATH}
    \end{subfigure}
    \caption{Trace-quality distributions under our Claude Sonnet 4.6 rubric-based judge. PoE produces more high-scoring traces than ADS on both datasets and remains closer to the standard teacher, suggesting better preservation of human-auditable reasoning.}
    \label{fig:trace-quality}
    \vspace{-4mm}
\end{figure}

\vspace{-3mm}
\section{Conclusion and Future Work}
\vspace{-3mm}
We introduced a game-theoretic view of distillation attacks and defenses in which a utility-constrained teacher interacts with an adaptive student. Once a value function is specified, the framework yields both an adaptive evaluation protocol for the student and a  template for teacher-side defense. Using a cheap proxy for example value, we further derived Product-of-Experts (PoE), a simple forward-pass-only defense that combines the teacher with a proxy student during generation.

Our findings suggest that progress on antidistillation should be judged against adaptive rather than passive students.
Across \texttt{GSM8K} and \texttt{MATH}, defenses that appear strong against passive students leak substantially more under adaptive evaluation. This changes the practical comparison between defenses: ADS appears stronger under passive evaluation, but its advantage narrows substantially under adaptive evaluation, while PoE is cheaper to run and preserves higher-quality reasoning traces.

Our framework deliberately focuses on adaptive reweighting, a simple and practical form of attacker adaptation that yields tractable response rules and implementable evaluation protocols. This focus is useful because it captures a capability that realistic distillers already have---choosing which released examples to emphasize---while remaining simple enough to connect defense design and evaluation in a single model. Natural extensions include richer adaptive attacks, broader model and task families, and stronger proxy choices for estimating distillation value. More broadly, we hope this perspective encourages future antidistillation work to specify not only how a teacher is defended, but also what adaptive student the defense is meant to withstand.

\bibliography{references}
\bibliographystyle{apalike}

\newpage
\appendix
\section*{Appendix}

\startcontents[appendix]
\begingroup
\setcounter{tocdepth}{2} %
\printcontents[appendix]{}{1}{}
\endgroup

\section{Proof of Theorem~\ref{thm:best-responses}}
\label{appx:proof}
We restate the theorem for convenience.

\bestresponsesthm*

Throughout this section, $\mathscr Y$ is finite for each context. Hence for each fixed context
$x\in\mathscr X$, all probability measures on $\mathscr Y$ can be identified with vectors in the
simplex $\Delta(\mathscr Y)$, all expectations are finite, and all optimization problems below
attain their optima on compact feasible sets.

For $\eta\in[0,+\infty)$ and a released policy $\pi_{\mathsf{rel}}$, define the exponential tilt:
\begin{equation}
    \pi_{\mathsf{eff},\eta}(y\mid x)
:=
\frac{\pi_{\mathsf{rel}}(y\mid x)e^{\eta v(x,y)}}
{\sum_{y'\in\mathscr Y}\pi_{\mathsf{rel}}(y'\mid x)e^{\eta v(x,y')}}, \quad \forall y \in \mathscr{Y}, \forall x \in \mathscr{X}.
\end{equation}
Likewise, for $\lambda\in[0,+\infty)$ and a nominal teacher $\pi_{\mathsf{ref}}$, define
\begin{equation}
    \pi_{\mathsf{rel},\lambda}(y\mid x)
:=
\frac{\pi_{\mathsf{ref}}(y\mid x)e^{-\lambda v(x,y)}}
{\sum_{y'\in\mathscr Y}\pi_{\mathsf{ref}}(y'\mid x)e^{-\lambda v(x,y')}}, \quad \forall y \in \mathscr{Y}, \forall x \in \mathscr{X}.
\end{equation}
For $\eta=+\infty$, we interpret $\pi_{\mathsf{eff},\eta}$ as an arbitrary policy supported on
\begin{equation}
    \arg\max_{y:\,\pi_{\mathsf{rel}}(y\mid x)>0} v(x,y),
\end{equation}
and for $\lambda=+\infty$, we interpret $\pi_{\mathsf{rel},\lambda}$ as an arbitrary policy supported on
\begin{equation}
    \arg\min_{y:\,\pi_{\mathsf{ref}}(y\mid x)>0} v(x,y).
\end{equation}
These are exactly the limiting policies of $\pi_{\mathsf{eff},\eta}$ and
$\pi_{\mathsf{rel},\lambda}$ as $\eta,\lambda\to +\infty$.

\subsection{Student best response}

We first prove the student part of the theorem.

\begin{proposition}
\label{prop:student-best-response}
Fix $\pi_{\mathsf{rel}}$ and $\rho\ge 0$. Then the optimization problem
\[
\sup\left\{
\mathbb E_{x\sim\mathscr D,\;y\sim\pi_{\mathsf{eff}}(\cdot\mid x)}[v(x,y)]
:
\pi_{\mathsf{eff}},\
\mathbb E_{x\sim\mathscr D}
\bigl[\mathrm{D}_{\mathrm{KL}}(\pi_{\mathsf{eff}}(\cdot\mid x)\|\pi_{\mathsf{rel}}(\cdot\mid x))\bigr]
\le \rho
\right\}
\]
admits an optimizer. Every optimizer is of the form $\pi_{\mathsf{eff},\eta}$ for some
$\eta\in[0,+\infty]$. Moreover, if the KL constraint is active, then $\eta<\infty$ and
\[
\mathbb E_{x\sim\mathscr D}
\bigl[\mathrm{D}_{\mathrm{KL}}(\pi_{\mathsf{eff},\eta}(\cdot\mid x)\|\pi_{\mathsf{rel}}(\cdot\mid x))\bigr]
=
\rho.
\]
\end{proposition}

\begin{proof}
Since $\mathscr Y$ is finite for each context, the feasible set
\[
\left\{
\pi_{\mathsf{eff}}
:
\mathbb E_{x\sim\mathscr D}
\bigl[\mathrm{D}_{\mathrm{KL}}(\pi_{\mathsf{eff}}(\cdot\mid x)\|\pi_{\mathsf{rel}}(\cdot\mid x))\bigr]
\le \rho
\right\}
\]
is compact, and the objective
$\pi_{\mathsf{eff}}\mapsto \mathbb E_{x\sim\mathscr D,\;y\sim\pi_{\mathsf{eff}}(\cdot\mid x)}[v(x,y)]$
is continuous. Hence an optimizer exists.

We first treat the case $\eta<\infty$. Fix $x \in \mathscr{X}$ and define:
\[
Z_{\pi_{\mathsf{rel}}}(x,\eta):=
\sum_{y\in\mathscr Y}\pi_{\mathsf{rel}}(y\mid x)e^{\eta v(x,y)}.
\]
Then for any $\pi\ll \pi_{\mathsf{rel}}$, i.e., $\pi$ is absolutely continuous w.r.t. $\pi_{\mathsf{rel}}$, we have:
\begin{align*}
&\mathbb E_{x\sim\mathscr D}
\bigl[\mathrm{D}_{\mathrm{KL}}(\pi(\cdot\mid x)\|\pi_{\mathsf{eff},\eta}(\cdot\mid x))\bigr] 
=
\mathbb E_{x\sim\mathscr D}
\left[
\sum_{y\in\mathscr Y}\pi(y\mid x)\log\frac{\pi(y\mid x)}{\pi_{\mathsf{eff},\eta}(y\mid x)}
\right] \\
&=
\mathbb E_{x\sim\mathscr D}
\left[
\sum_{y\in\mathscr Y}\pi(y\mid x)\log\frac{\pi(y\mid x)}{\pi_{\mathsf{rel}}(y\mid x)e^{\eta v(x,y)}/Z_{\pi_{\mathsf{rel}}}(x,\eta)}
\right] \\
&=
\mathbb E_{x\sim\mathscr D}
\bigl[\mathrm{D}_{\mathrm{KL}}(\pi(\cdot\mid x)\|\pi_{\mathsf{rel}}(\cdot\mid x))\bigr]
-\eta\,\mathbb E_{x\sim\mathscr D,\;y\sim\pi(\cdot\mid x)}[v(x,y)]
+\mathbb E_{x\sim\mathscr D}\bigl[\log Z_{\pi_{\mathsf{rel}}}(x,\eta)\bigr].
\end{align*}
Rearranging gives
\begin{align*}
    &\mathbb E_{x\sim\mathscr D,\;y\sim\pi(\cdot\mid x)}[v(x,y)]
-\frac{1}{\eta}
\mathbb E_{x\sim\mathscr D}
\bigl[\mathrm{D}_{\mathrm{KL}}(\pi(\cdot\mid x)\|\pi_{\mathsf{rel}}(\cdot\mid x))\bigr]
\\
&=
\frac{1}{\eta}\mathbb E_{x\sim\mathscr D}\bigl[\log Z_{\pi_{\mathsf{rel}}}(x,\eta)\bigr]
-\frac{1}{\eta}
\mathbb E_{x\sim\mathscr D}
\bigl[\mathrm{D}_{\mathrm{KL}}(\pi(\cdot\mid x)\|\pi_{\mathsf{eff},\eta}(\cdot\mid x))\bigr].
\end{align*}
Therefore, for every $\eta\in(0,+\infty)$, the unique maximizer of the penalized problem
\[
\sup_{\pi}
\left\{
\mathbb E_{x\sim\mathscr D,\;y\sim\pi(\cdot\mid x)}[v(x,y)]
-\frac{1}{\eta}
\mathbb E_{x\sim\mathscr D}
\bigl[\mathrm{D}_{\mathrm{KL}}(\pi(\cdot\mid x)\|\pi_{\mathsf{rel}}(\cdot\mid x))\bigr]
\right\}
\]
is $\pi_{\mathsf{eff},\eta}$.

We now pass to the constrained problem. Since the feasible set has nonempty relative interior
(for example, $\pi=\pi_{\mathsf{rel}}$ is feasible), strong duality holds for this finite-dimensional
convex program. Thus there exists a Lagrange multiplier corresponding to the KL constraint. If the
constraint is active at the optimum, then the KKT conditions imply that any optimizer must maximize a
penalized objective of the form above for some finite $\eta>0$, hence must equal
$\pi_{\mathsf{eff},\eta}$, and complementary slackness gives
\[
\mathbb E_{x\sim\mathscr D}
\bigl[\mathrm{D}_{\mathrm{KL}}(\pi_{\mathsf{eff},\eta}(\cdot\mid x)\|\pi_{\mathsf{rel}}(\cdot\mid x))\bigr]
=
\rho.
\]

It remains to consider the inactive case. If the KL constraint is inactive, then the optimum is the
same as that of the unconstrained linear functional
$\pi\mapsto \mathbb E_{x\sim\mathscr D,\;y\sim\pi(\cdot\mid x)}[v(x,y)]$
over the simplex subject only to $\pi\ll \pi_{\mathsf{rel}}$. Such an optimizer is any policy
supported on
\[
\arg\max_{y:\,\pi_{\mathsf{rel}}(y\mid x)>0} v(x,y).
\]
This is exactly our definition of $\pi_{\mathsf{eff},\infty}$.

Hence every optimizer is of the form $\pi_{\mathsf{eff},\eta}$ for some $\eta\in[0,+\infty]$, with
$\eta<\infty$ and
\[
\mathbb E_{x\sim\mathscr D}
\bigl[\mathrm{D}_{\mathrm{KL}}(\pi_{\mathsf{eff},\eta}(\cdot\mid x)\|\pi_{\mathsf{rel}}(\cdot\mid x))\bigr]
=
\rho
\]
whenever the KL constraint is active.
\end{proof}

\subsection{Teacher best response}

The teacher part is the corresponding minimization problem.

\begin{proposition}
\label{prop:teacher-best-response}
Fix $\pi_{\mathsf{ref}}$ and $\varepsilon\ge 0$. Then the optimization problem
\[
\inf\left\{
\mathbb E_{x\sim\mathscr D,\;y\sim\pi_{\mathsf{rel}}(\cdot\mid x)}[v(x,y)]
:
\pi_{\mathsf{rel}},\
\mathbb E_{x\sim\mathscr D}
\bigl[\mathrm{D}_{\mathrm{KL}}(\pi_{\mathsf{rel}}(\cdot\mid x)\|\pi_{\mathsf{ref}}(\cdot\mid x))\bigr]
\le \varepsilon
\right\}
\]
admits an optimizer. Every optimizer is of the form $\pi_{\mathsf{rel},\lambda}$ for some
$\lambda\in[0,+\infty]$. Moreover, if the KL constraint is active, then $\lambda<\infty$ and
\[
\mathbb E_{x\sim\mathscr D}
\bigl[\mathrm{D}_{\mathrm{KL}}(\pi_{\mathsf{rel},\lambda}(\cdot\mid x)\|\pi_{\mathsf{ref}}(\cdot\mid x))\bigr]
=
\varepsilon.
\]
\end{proposition}

\begin{proof}
The argument is the same as for Proposition~\ref{prop:student-best-response}, applied to the
minimization problem. Since $\mathscr Y$ is finite for each context, the feasible set
\[
\left\{
\pi_{\mathsf{rel}}
:
\mathbb E_{x\sim\mathscr D}
\bigl[\mathrm{D}_{\mathrm{KL}}(\pi_{\mathsf{rel}}(\cdot\mid x)\|\pi_{\mathsf{ref}}(\cdot\mid x))\bigr]
\le \varepsilon
\right\}
\]
is compact, and the objective
$\pi_{\mathsf{rel}}\mapsto \mathbb E_{x\sim\mathscr D,\;y\sim\pi_{\mathsf{rel}}(\cdot\mid x)}[v(x,y)]$
is continuous, so an optimizer exists.

For $\lambda\in(0,+\infty)$, define
\[
Z_{\pi_{\mathsf{ref}}}(x,\lambda):=
\sum_{y\in\mathscr Y}\pi_{\mathsf{ref}}(y\mid x)e^{-\lambda v(x,y)}.
\]
Then for any $\pi\ll \pi_{\mathsf{ref}}$,
\begin{align*}
&\mathbb E_{x\sim\mathscr D}
\bigl[\mathrm{D}_{\mathrm{KL}}(\pi(\cdot\mid x)\|\pi_{\mathsf{rel},\lambda}(\cdot\mid x))\bigr] \\
&=
\mathbb E_{x\sim\mathscr D}
\left[
\sum_{y\in\mathscr Y}\pi(y\mid x)\log\frac{\pi(y\mid x)}{\pi_{\mathsf{rel},\lambda}(y\mid x)}
\right] \\
&=
\mathbb E_{x\sim\mathscr D}
\left[
\sum_{y\in\mathscr Y}\pi(y\mid x)\log\frac{\pi(y\mid x)}{\pi_{\mathsf{ref}}(y\mid x)e^{-\lambda v(x,y)}/Z_{\pi_{\mathsf{ref}}}(x,\lambda)}
\right] \\
&=
\mathbb E_{x\sim\mathscr D}
\bigl[\mathrm{D}_{\mathrm{KL}}(\pi(\cdot\mid x)\|\pi_{\mathsf{ref}}(\cdot\mid x))\bigr]
+\lambda\,\mathbb E_{x\sim\mathscr D,\;y\sim\pi(\cdot\mid x)}[v(x,y)]
+\mathbb E_{x\sim\mathscr D}\bigl[\log Z_{\pi_{\mathsf{ref}}}(x,\lambda)\bigr].
\end{align*}
Rearranging gives
\begin{align*}
  &\mathbb E_{x\sim\mathscr D,\;y\sim\pi(\cdot\mid x)}[v(x,y)]
+\frac{1}{\lambda}
\mathbb E_{x\sim\mathscr D}
\bigl[\mathrm{D}_{\mathrm{KL}}(\pi(\cdot\mid x)\|\pi_{\mathsf{ref}}(\cdot\mid x))\bigr]\\
&=
-\frac{1}{\lambda}\mathbb E_{x\sim\mathscr D}\bigl[\log Z_{\pi_{\mathsf{ref}}}(x,\lambda)\bigr]
+\frac{1}{\lambda}
\mathbb E_{x\sim\mathscr D}
\bigl[\mathrm{D}_{\mathrm{KL}}(\pi(\cdot\mid x)\|\pi_{\mathsf{rel},\lambda}(\cdot\mid x))\bigr].  
\end{align*}
Thus, for every $\lambda\in(0,+\infty)$, the unique minimizer of
\[
\inf_{\pi}
\left\{
\mathbb E_{x\sim\mathscr D,\;y\sim\pi(\cdot\mid x)}[v(x,y)]
+\frac{1}{\lambda}
\mathbb E_{x\sim\mathscr D}
\bigl[\mathrm{D}_{\mathrm{KL}}(\pi(\cdot\mid x)\|\pi_{\mathsf{ref}}(\cdot\mid x))\bigr]
\right\}
\]
is $\pi_{\mathsf{rel},\lambda}$.

As above, strong duality holds because the feasible set has nonempty relative interior
(for example, $\pi=\pi_{\mathsf{ref}}$ is feasible). Hence if the KL constraint is active, the KKT
conditions imply that any optimizer must equal $\pi_{\mathsf{rel},\lambda}$ for some finite
$\lambda>0$, and complementary slackness gives
\[
\mathbb E_{x\sim\mathscr D}
\bigl[\mathrm{D}_{\mathrm{KL}}(\pi_{\mathsf{rel},\lambda}(\cdot\mid x)\|\pi_{\mathsf{ref}}(\cdot\mid x))\bigr]
=
\varepsilon.
\]

If the KL constraint is inactive, then the optimum coincides with that of minimizing the linear
functional $\pi\mapsto \mathbb E_{x\sim\mathscr D,\;y\sim\pi(\cdot\mid x)}[v(x,y)]$
over all $\pi\ll \pi_{\mathsf{ref}}$. Any optimizer is then supported on
\[
\arg\min_{y:\,\pi_{\mathsf{ref}}(y\mid x)>0} v(x,y),
\]
which is exactly our definition of $\pi_{\mathsf{rel},\infty}$.

Hence every optimizer is of the form $\pi_{\mathsf{rel},\lambda}$ for some
$\lambda\in[0,+\infty]$, with $\lambda<\infty$ and
\[
\mathbb E_{x\sim\mathscr D}
\bigl[\mathrm{D}_{\mathrm{KL}}(\pi_{\mathsf{rel},\lambda}(\cdot\mid x)\|\pi_{\mathsf{ref}}(\cdot\mid x))\bigr]
=
\varepsilon
\]
whenever the KL constraint is active.
\end{proof}

\subsection{Proof of the theorem}

\begin{proof}[Proof of Theorem~\ref{thm:best-responses}]
The student statement is Proposition~\ref{prop:student-best-response}, and the teacher statement
is Proposition~\ref{prop:teacher-best-response}. The formulas displayed in the theorem are exactly
the definitions of $\pi_{\mathsf{eff},\eta}$ and $\pi_{\mathsf{rel},\lambda}$, with the cases
$\eta=\infty$ and $\lambda=\infty$ understood as the limiting policies described above.
\end{proof}

\section{Stackelberg equilibrium of the distillation game}
\label{app:stackelberg}

In this section, we analyze the full game
\[
\mathscr{V}(\varepsilon,\rho)
:=
\inf_{\pi_{\mathsf{rel}}\in\Pi_\varepsilon(\pi_{\mathsf{ref}})}
\sup_{\pi_{\mathsf{eff}}\in\Pi_\rho(\pi_{\mathsf{rel}})}
\mathbb E_{x\sim\mathscr D,\;y\sim\pi_{\mathsf{eff}}(\cdot\mid x)}
[v(x,y)].
\]
For the finite-dimensional arguments below, we work in the finite-$\mathscr Z$
setting. Since $\mathscr D$ is fixed throughout, this is equivalent to working
with the induced joint laws of the policies.

\subsection{Variational form of the student's problem}

For fixed $\pi_{\mathsf{rel}}$, define
\[
\Psi(\pi_{\mathsf{rel}})
:=
\sup_{\pi_{\mathsf{eff}}\in\Pi_\rho(\pi_{\mathsf{rel}})}
\mathbb E_{x\sim\mathscr D,\;y\sim\pi_{\mathsf{eff}}(\cdot\mid x)}
[v(x,y)].
\]
For $\eta>0$, let
\[
Z_{\pi_{\mathsf{rel}}}(x,\eta)
:=
\sum_{y\in\mathscr Y}
\pi_{\mathsf{rel}}(y\mid x)e^{\eta v(x,y)}.
\]

\begin{lemma}
\label{lem:inner-dual}
For every $\pi_{\mathsf{rel}}$,
\[
\Psi(\pi_{\mathsf{rel}})
=
\inf_{\eta>0}
\frac{
\rho
+
\mathbb E_{x\sim\mathscr D}
\bigl[\log Z_{\pi_{\mathsf{rel}}}(x,\eta)\bigr]
}{\eta}.
\]
The boundary cases are interpreted by limits: $\eta\downarrow 0$ when
$\rho=0$, and $\eta\to\infty$ when the student's KL constraint is inactive.
\end{lemma}

\begin{proof}
Fix $\eta>0$ and define
\[
\pi_{\mathsf{eff},\eta}(y\mid x)
:=
\frac{
\pi_{\mathsf{rel}}(y\mid x)e^{\eta v(x,y)}
}{
Z_{\pi_{\mathsf{rel}}}(x,\eta)
}.
\]
For every $\pi_{\mathsf{eff}}\ll\pi_{\mathsf{rel}}$,
\begin{align*}
&
\mathbb E_x
\bigl[
\mathrm{D}_{\mathrm{KL}}
(
\pi_{\mathsf{eff}}(\cdot\mid x)
\|
\pi_{\mathsf{eff},\eta}(\cdot\mid x)
)
\bigr]
\\
&=
\mathbb E_x
\bigl[
\mathrm{D}_{\mathrm{KL}}
(
\pi_{\mathsf{eff}}(\cdot\mid x)
\|
\pi_{\mathsf{rel}}(\cdot\mid x)
)
\bigr]
-
\eta\,
\mathbb E_{x,y\sim\pi_{\mathsf{eff}}}[v(x,y)]
+
\mathbb E_x
\bigl[
\log Z_{\pi_{\mathsf{rel}}}(x,\eta)
\bigr].
\end{align*}
Since the left-hand side is nonnegative, every
$\pi_{\mathsf{eff}}\in\Pi_\rho(\pi_{\mathsf{rel}})$ satisfies
\[
\mathbb E_{x,y\sim\pi_{\mathsf{eff}}}[v(x,y)]
\le
\frac{
\rho
+
\mathbb E_x
[
\log Z_{\pi_{\mathsf{rel}}}(x,\eta)
]
}{\eta}.
\]
Taking the supremum over feasible $\pi_{\mathsf{eff}}$ and then the infimum
over $\eta>0$ gives one direction.

If the KL constraint is active, Proposition~\ref{prop:student-best-response}
gives some finite $\eta^\star>0$ for which
\[
\pi_{\mathsf{eff}}^\star
=
\pi_{\mathsf{eff},\eta^\star},
\qquad
\mathbb E_x
\bigl[
\mathrm{D}_{\mathrm{KL}}
(
\pi_{\mathsf{eff}}^\star(\cdot\mid x)
\|
\pi_{\mathsf{rel}}(\cdot\mid x)
)
\bigr]
=
\rho.
\]
Substitution into the identity above makes the bound tight.

If $\rho=0$, then
\[
\lim_{\eta\downarrow 0}
\frac{1}{\eta}
\log Z_{\pi_{\mathsf{rel}}}(x,\eta)
=
\mathbb E_{y\sim\pi_{\mathsf{rel}}(\cdot\mid x)}[v(x,y)].
\]
If the KL constraint is inactive, then
\[
\lim_{\eta\to\infty}
\frac{1}{\eta}
\log Z_{\pi_{\mathsf{rel}}}(x,\eta)
=
\max_{y:\,\pi_{\mathsf{rel}}(y\mid x)>0} v(x,y).
\]
Since $\mathscr X$ is finite, averaging over $x\sim\mathscr D$ yields the
remaining boundary cases.
\end{proof}

\subsection{Existence and characterization of an optimal pair}

Next, the variational form below gives a direct reduction of the full
Stackelberg game.

\begin{lemma}
\label{lem:game-reduction}
The game value satisfies
\[
\mathscr V(\varepsilon,\rho)
=
\inf_{\eta>0}
\inf_{\pi_{\mathsf{rel}}\in\Pi_\varepsilon(\pi_{\mathsf{ref}})}
\frac{
\rho
+
\mathbb E_{x\sim\mathscr D}
\left[
\log
\sum_{y\in\mathscr Y}
\pi_{\mathsf{rel}}(y\mid x)e^{\eta v(x,y)}
\right]
}{\eta}.
\]
\end{lemma}

\begin{proof}
By Lemma~\ref{lem:inner-dual},
\[
\mathscr V(\varepsilon,\rho)
=
\inf_{\pi_{\mathsf{rel}}\in\Pi_\varepsilon(\pi_{\mathsf{ref}})}
\inf_{\eta>0}
\frac{
\rho
+
\mathbb E_x
[
\log Z_{\pi_{\mathsf{rel}}}(x,\eta)
]
}{\eta}.
\]
The two infima may be exchanged because both sides are the infimum of the
same function over
$\Pi_\varepsilon(\pi_{\mathsf{ref}})\times(0,\infty)$.
\end{proof}

We now show that the game admits a Stackelberg-optimal teacher--student pair.

\begin{theorem}
\label{thm:optimal-pair}
There exists
$\pi_{\mathsf{rel}}^\star\in\Pi_\varepsilon(\pi_{\mathsf{ref}})$
and
$\pi_{\mathsf{eff}}^\star\in\Pi_\rho(\pi_{\mathsf{rel}}^\star)$
such that
\[
\mathscr V(\varepsilon,\rho)
=
\mathbb E_{x\sim\mathscr D,\;
y\sim\pi_{\mathsf{eff}}^\star(\cdot\mid x)}
[v(x,y)].
\]

Suppose moreover that the variational problem in
Lemma~\ref{lem:inner-dual}, evaluated at
$\pi_{\mathsf{rel}}^\star$, is attained at some
$\eta^\star\in(0,\infty)$. Then
\[
\pi_{\mathsf{eff}}^\star(y\mid x)
=
\frac{
\pi_{\mathsf{rel}}^\star(y\mid x)e^{\eta^\star v(x,y)}
}{
Z^\star(x)
},
\qquad
Z^\star(x)
:=
\sum_{y'}
\pi_{\mathsf{rel}}^\star(y'\mid x)e^{\eta^\star v(x,y')},
\]
and $\pi_{\mathsf{rel}}^\star$ solves
\[
\pi_{\mathsf{rel}}^\star
\in
\arg\min_{\pi_{\mathsf{rel}}\in\Pi_\varepsilon(\pi_{\mathsf{ref}})}
\mathbb E_{x\sim\mathscr D}
\left[
\log
\sum_y
\pi_{\mathsf{rel}}(y\mid x)e^{\eta^\star v(x,y)}
\right].
\]
\end{theorem}

\begin{proof}
Define
\[
\Phi(\pi_{\mathsf{rel}})
:=
\sup_{\pi_{\mathsf{eff}}\in\Pi_\rho(\pi_{\mathsf{rel}})}
\mathbb E_{x,y\sim\pi_{\mathsf{eff}}}[v(x,y)].
\]
Because $\mathscr Z$ is finite,
$\Pi_\varepsilon(\pi_{\mathsf{ref}})$ is compact. The correspondence
\[
\pi_{\mathsf{rel}}
\mapsto
\Pi_\rho(\pi_{\mathsf{rel}})
\]
is nonempty and compact-valued, and is continuous on the finite-dimensional
simplex. Indeed, joint lower semicontinuity of relative entropy gives a closed
graph, while feasible policies on the boundary can be approximated by mixing
them with the nearby released policy. Hence Berge's maximum theorem implies
that $\Phi$ is continuous.

Therefore $\Phi$ attains its minimum on
$\Pi_\varepsilon(\pi_{\mathsf{ref}})$; let
$\pi_{\mathsf{rel}}^\star$ be a minimizer. For this fixed released policy,
$\Pi_\rho(\pi_{\mathsf{rel}}^\star)$ is compact and the objective is
continuous, so a maximizing
$\pi_{\mathsf{eff}}^\star$ also exists. This proves existence.

Now suppose the dual infimum at $\pi_{\mathsf{rel}}^\star$ is attained at
$\eta^\star\in(0,\infty)$. By Lemma~\ref{lem:inner-dual}, the equality
condition gives
\[
\pi_{\mathsf{eff}}^\star(y\mid x)
\propto
\pi_{\mathsf{rel}}^\star(y\mid x)
e^{\eta^\star v(x,y)}.
\]

Finally, define
\[
F(\pi_{\mathsf{rel}},\eta)
:=
\frac{
\rho
+
\mathbb E_x[
\log Z_{\pi_{\mathsf{rel}}}(x,\eta)
]
}{\eta}.
\]
Since
\[
F(\pi_{\mathsf{rel}}^\star,\eta^\star)
=
\mathscr V(\varepsilon,\rho)
\]
and, for every feasible $\pi_{\mathsf{rel}}$,
\[
F(\pi_{\mathsf{rel}},\eta^\star)
\ge
\inf_{\eta>0}F(\pi_{\mathsf{rel}},\eta)
=
\Phi(\pi_{\mathsf{rel}})
\ge
\mathscr V(\varepsilon,\rho),
\]
the policy $\pi_{\mathsf{rel}}^\star$ minimizes
$F(\cdot,\eta^\star)$. Since $\rho$ and $\eta^\star$ are fixed, this is
equivalent to the stated minimization problem.
\end{proof}

\section{Additional experiments and experimental details} \label{sec:add-expr}
In this section, we provide details of our experiments as well as some additional experiments. 

{
\subsection{Experimental details}
\label{app:experimental-details}

\paragraph{Datasets and splits.}
We evaluate on GSM8K and MATH. For GSM8K, we use the standard training split and materialize
disjoint train, holdout, and test subsets with fixed random seeds; the runs underlying the main
figures use 5,238/2,246/1,319 prompts for train/holdout/test, respectively. For MATH, we
construct a pooled dataset from the standard subject-area subsets and materialize 5,000/2,500/5,000
train/holdout/test examples.

\paragraph{Teacher, proxy, and student models.}
Across all experiments, the reference teacher is DeepSeek-R1-Distill-Qwen-7B, the proxy student
used for teacher-side is Qwen2.5-3B, and the final distilled student is Llama-3.2-3B. We use \texttt{bfloat16} inference/training with
\texttt{flash\_attention\_2}~\citep{dao2022flashattention}. This proxy/final-student mismatch is intentional and follows prior work such as~\cite{savani2025antidistillation}: the teacher is designed
against an internal proxy, but evaluated against a different attacker architecture, so the setting is a
transfer-style antidistillation evaluation rather than a matched white-box one.

\paragraph{Teacher generation.}
Teacher traces are generated with temperature 0.6 and top-$p$ 0.95, with prompts truncated to 512
tokens and generations capped at 1,024 new tokens. We also append a short answer-forcing suffix
that allocates up to 32 additional tokens for the boxed final answer. For the defense families, we
compare a standard teacher, ADS, and PoE. In the representative tables, we report GSM8K ADS at
$\lambda=0.052$ and PoE at $\gamma=0.65$, and MATH ADS at $\lambda=0.08$ and PoE at
$\gamma=0.75$; broader sweep results are shown in the trade-off plots.

For GSM8K, ADS is swept on $\lambda \in \{0.043, 0.044, 0.046, 0.047, 0.049, 0.050, 0.052, 0.055, 0.060\}$ and PoE is swept on $\gamma \in \{0.50, 0.55, 0.60, 0.63, 0.65, 0.68, 0.70, 0.73, 0.75\}$. For MATH, we report the sweeps: ADS is swept on $\lambda \in \{0.07, 0.075, 0.0775, 0.08, 0.085, 0.0875, 0.09\}$ and PoE is swept on $\gamma \in \{0.70, 0.725, 0.75, 0.80, 0.825, 0.85, 0.90\}$.

\paragraph{Student distillation.}
All students are trained with the same supervised fine-tuning pipeline and differ only in how
training traces are weighted. We fine-tune the final student with LoRA for 3 epochs using rank
128, LoRA $\alpha=128$, dropout 0, learning rate $5\times 10^{-4}$, weight decay 0.1, gradient
clipping at 1.0, cosine decay with warmup ratio 0.03, per-device batch size 6, gradient
accumulation 4, and maximum sequence length 2,048. The passive student uses uniform weighting
over released traces. The adaptive student uses sharpness parameter $\beta_s=0.5$. 

\paragraph{Reporting and hardware.}
Our main metrics are teacher accuracy and downstream student accuracy under passive and adaptive
evaluation. All reported numbers are averaged over three seeds (123, 456, 789), and we report
standard errors across seeds. Experiments were run on NVIDIA B200, H200, and A100 machines.
}

\subsection{Distillation using frontier models}
\label{app:distillation_signal}

We also evaluate whether traces from commercial frontier models (GPT-5.4 mini, Claude Sonnet 4.6, Gemini 3 Flash Preview) can serve as effective
supervision for a smaller open-weight student. Figure~\ref{fig:frontier} shows that fine-tuning
Llama-3.2-3B on frontier-model traces yields dramatic gains over both the untuned base
model and a baseline trained only on question--answer pairs. On GSM8K, the student
improves from $0.76\%$ in the base setting and $3.43\%$ with Q\&A-only supervision to
$51.75\%$, $53.57\%$, and $54.78\%$ when distilled from Gemini, Claude, and GPT traces,
respectively. On MATH, the same pattern holds, with accuracy rising from $1.01\%$ at base
and $11.25\%$ with Q\&A-only supervision to $12.09\%$, $16.33\%$, and $14.83\%$ using
Gemini, Claude, and GPT traces. These results show that frontier-model traces contain
substantial reusable learning signal beyond final answers alone, while also indicating
meaningful variation across source models in how effectively that signal transfers to the
student.

\begin{figure}
    \centering
    \begin{subfigure}[b]{0.48\linewidth}
        \centering
        \includegraphics[width=\linewidth]{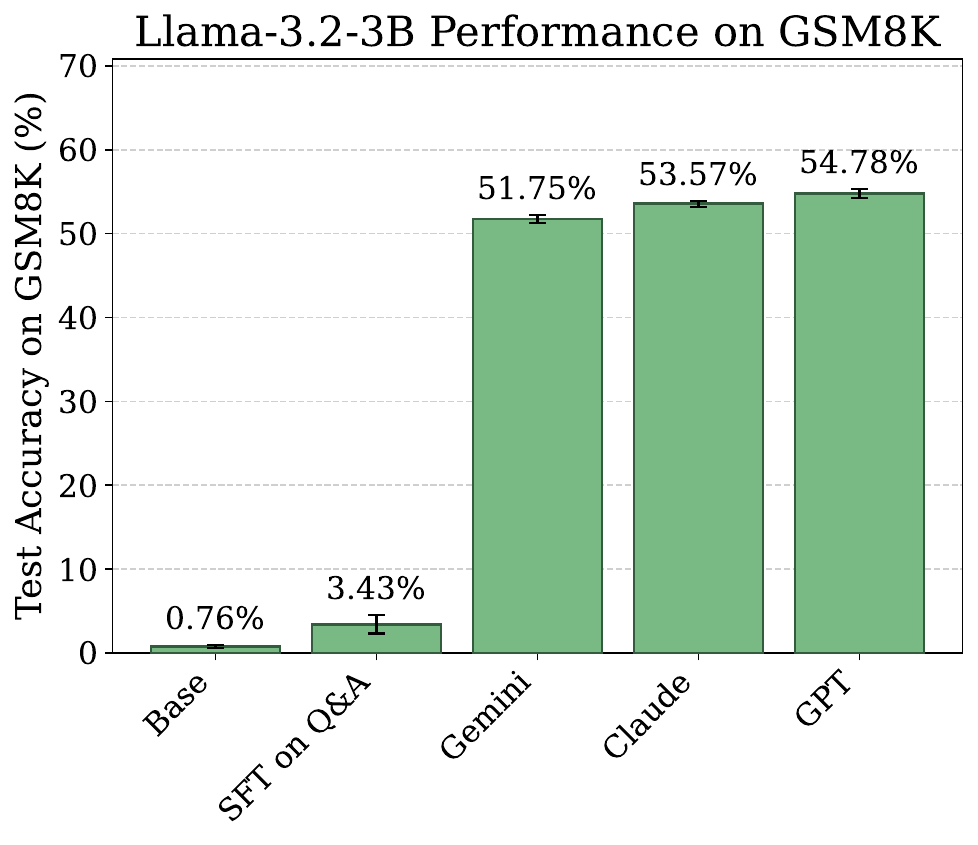}
    \end{subfigure}
    \hfill
    \begin{subfigure}[b]{0.48\linewidth}
        \centering
        \includegraphics[width=\linewidth]{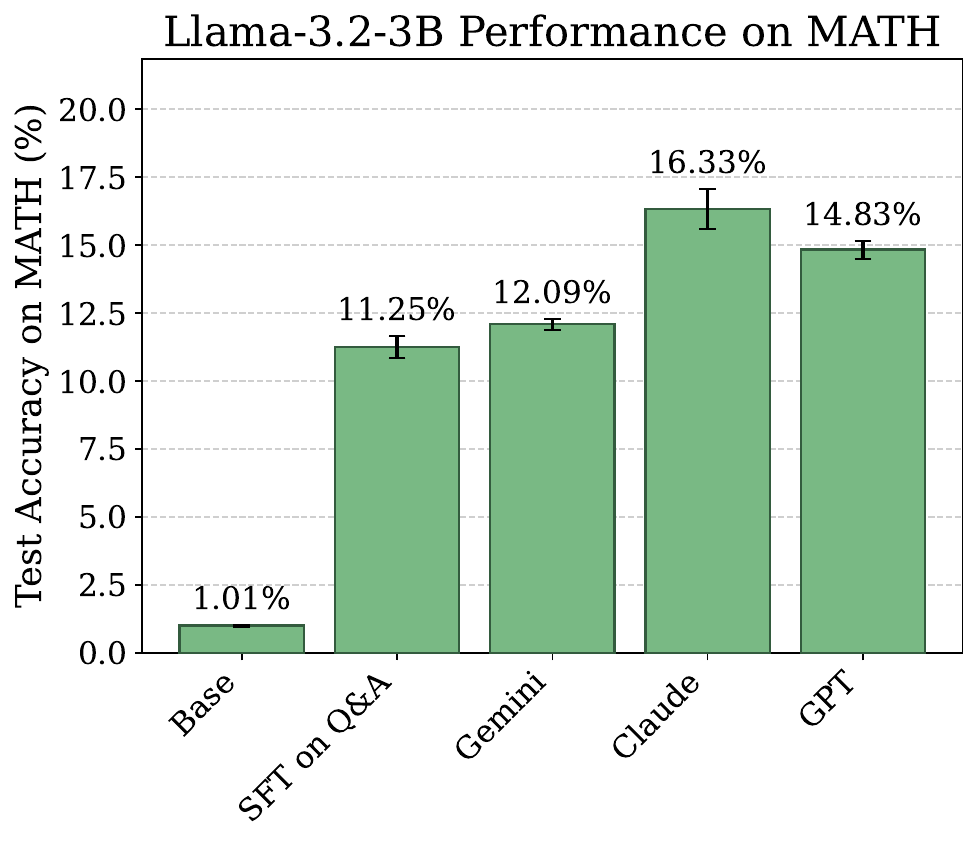}
    \end{subfigure}
    \caption{Student accuracy after distillation from commercial frontier-model (GPT-5.4 mini, Claude Sonnet 4.6, Gemini 3 Flash Preview) outputs under different exposure formats. Summary traces provide less information than full reasoning traces but still yield nontrivial student recovery, while answer-only outputs are generally less informative. Results are averaged over 3 seeds; error bars show standard error of the mean.}
    \label{fig:frontier}
\end{figure}

\subsection{LLM-as-a-Judge prompt for trace auditability}
\label{sec:auditability-prompt}

To evaluate the utility of the generated reasoning traces, we employ Claude Sonnet 4.6 as an automated judge. We provide the judge with the original math problem and the model's generated trace, instructing it to evaluate the trace based strictly on its \emph{auditability}. The exact system prompt and rubric provided to the model are detailed below.
Each trace is judged independently, without revealing the defense type. The judge receives the problem statement, final answer, and reasoning trace. The order of examples is randomized before scoring.

\vspace{1em}
\begin{tcolorbox}[colback=gray!5!white, colframe=gray!50!black, title=\textbf{System Prompt: Trace Auditability}, arc=2mm]
\small
You are a judge for the quality of reasoning traces. You will be given a math problem and a reasoning trace. Score the trace on a 1--5 scale using the rubric below.

\vspace{0.5em}
\textbf{\#\#\# Rubric: Auditability of Reasoning}

\vspace{0.5em}
\textbf{Criteria}: Can a human reader audit the reasoning in this trace? For each step, assess whether it is (a) identifiable as a distinct reasoning step, (b) relevant to the problem, and (c) checkable against the previous step or the problem statement. Verbose but coherent traces and concise traces should score equally if both are auditable.

\begin{itemize}[leftmargin=*, itemsep=0.5em]
    \item \textbf{Score 1} -- No auditable reasoning. The trace is dominated by non-reasoning content (repeated tokens, garbled text, or irrelevant material). A reader cannot identify any checkable steps.
    
    \item \textbf{Score 2} -- Few auditable steps. Some reasoning is present but is interleaved with substantial non-reasoning content (filler tokens, irrelevant tangents, or corrupted text) making it unclear which parts to trust.
    
    \item \textbf{Score 3} -- Partially auditable. The core reasoning steps are identifiable but some steps lack clear justification, or the reader must ignore non-trivial amounts of irrelevant content to follow the argument.
    
    \item \textbf{Score 4} -- Mostly auditable. Nearly every step is identifiable, relevant, and checkable. Minor issues such as a redundant restatement or one unclear transition do not prevent verification.
    
    \item \textbf{Score 5} -- Fully auditable. Every step is identifiable, relevant to the problem, and independently checkable. The trace may be long or short --- what matters is that no step requires guesswork to verify.
\end{itemize}

\end{tcolorbox}

To sanity-check the automated judge, we collected human ratings for 30 traces,
sampled evenly across teacher conditions: 10 Standard, 10 PoE, and 10 ADS.
Human raters used the same 1--5 trace-auditability rubric as the automated judge.
We report quadratic-weighted Cohen's $\kappa$, a chance-corrected agreement
measure for ordinal ratings that penalizes larger disagreements more heavily~\citep{cohen1968weighted,fleiss1973equivalence}. We also report mean absolute error,
$\frac{1}{n}\sum_{i=1}^n |j_i-h_i|$, and mean signed error,
$\frac{1}{n}\sum_{i=1}^n (j_i-h_i)$, where $h_i$ is the human rating and
$j_i$ is the judge rating. The judge achieves $\kappa=0.76$, mean absolute
error $0.40$, and mean signed error $+0.13$, indicating strong agreement with
only a small upward offset on the 1--5 scale.

\subsection{Ablation Studies}
\label{app:ablations}
\paragraph{Simple filtering of bad reasoning traces.}
We remove traces containing characters from an unexpected script or a token
repeated more than eight times consecutively, and retrain the passive student
on the remaining traces. Our adaptive student always trains on the full,
unfiltered release and uses neither rule. Table~\ref{tab:filtering} shows that
simple filtering improves passive SFT but does not match our attack.

\begin{table}[t]
\centering
\setlength{\tabcolsep}{3pt}
\resizebox{\linewidth}{!}{%
\begin{tabular}{lcccccc}
\toprule
\textbf{Teacher}
& \textbf{Teacher acc.} (\%)
& \textbf{Traces removed}
& \shortstack{\textbf{Passive}\\all (\%)}
& \shortstack{\textbf{Passive}\\filtered (\%)}
& \shortstack{\textbf{Adaptive}\\all (\%)}
& \textbf{$\Delta$} (pp) \\
\midrule
Standard
& $87.22\pm0.04$
& $1\;(0.02\%)$
& $57.24\pm0.42$
& $58.01\pm0.54$
& $56.24\pm0.17$
& $-1.77$ \\
ADS ($\lambda=0.055$)
& $82.65\pm0.02$
& $578\;(11.03\%)$
& $32.70\pm0.30$
& $35.28\pm0.34$
& $51.01\pm1.04$
& $+15.73$ \\
PoE ($\gamma=0.70$)
& $83.47\pm0.07$
& $213\;(4.07\%)$
& $31.05\pm0.20$
& $36.51\pm0.15$
& $49.20\pm0.90$
& $+12.69$ \\
\bottomrule
\end{tabular}%
}
\caption{GSM8K filtering ablation on regenerated traces. ``Passive filtered''
uses the mechanical filters before training; ``Adaptive all'' is our attack
trained on the original, unfiltered traces without either filter.
$\Delta$ is Adaptive all minus Passive filtered, computed from the displayed
means. Accuracies are mean $\pm$ standard error over three fresh seeds.
This experiment uses different defense strengths and regenerated traces from
Table~1 and should not be read as a reproduction of that table.}
\label{tab:filtering}
\end{table}
Filtering improves passive student accuracy by $2.58$ percentage points for ADS and
$5.46$ points for PoE. Nevertheless, our adaptive attack, without filtering, exceeds
these filtered-passive baselines by $15.73$ and $12.69$ points, respectively.
Thus, the two simple rules do not reproduce the capability recovered by
adaptive weighting.

\paragraph{Variations of teacher-proxy-student models,}
We test sensitivity to the defense proxy by replacing Qwen2.5-3B with
Qwen2.5-1.5B, while keeping the DeepSeek-R1-Distill-Qwen-7B teacher
and Llama-3.2-3B final student fixed. Table~\ref{tab:proxy-size}
reports the GSM8K results.

\begin{table}[t]
\centering
\begin{tabular}{lccccc}
\toprule
\textbf{Defense} & \textbf{Strength} & \textbf{Teacher acc.} (\%) & \textbf{Passive} (\%) &
\textbf{Adaptive} (\%) & \textbf{Gain} (pp) \\
\midrule
ADS & $\lambda=0.052$ & 21.42 & 1.08 & 1.16 & 0.08 \\
ADS & $\lambda=0.027$ & 81.56 & 39.78 & 51.70 & 11.92 \\
PoE & $\gamma=0.65$  & 78.00 & 45.24 & 52.94 & 7.70 \\
\bottomrule
\end{tabular}
\caption{Proxy-size ablation on GSM8K. The teacher--defense-proxy--final-student
triple is DeepSeek-R1-Distill-Qwen-7B / Qwen2.5-1.5B / Llama-3.2-3B.
Only the defense proxy changes from the main experiment, which uses
Qwen2.5-3B.}
\label{tab:proxy-size}
\end{table}

The ADS strength selected with the original 3B proxy does not
transfer: at $\lambda=0.052$, teacher accuracy falls to 21.42\%.
Retuning to $\lambda=0.027$ restores teacher accuracy to 81.56\%,
at which point the adaptive student exceeds the passive student
by 11.92 percentage points. PoE remains usable at its original
$\gamma=0.65$, although teacher accuracy decreases to 78.00\%;
the adaptive student exceeds the passive student by 7.70 points.
Thus, proxy size affects defense calibration, while adaptive
evaluation still reveals additional leakage at the usable
operating points.

\paragraph{Evaluation beyond ADS defense.}
We argue that anti-distillation methods should be evaluated jointly by the tradeoff of student-teacher accuracy, trace quality for human users, and generation cost. To show the importance of this viewpoint beyond \citep{savani2025antidistillation}, we consider a recent work \citep{ma2026protecting} on antidistillation using trace rewriting. The authors report that optimized rewriting changes GSM8K teacher accuracy from 88.8\% to 91.6\% and Llama-3.2-3B student accuracy from 57.7\% to 8.6\%. Also, their approach is very cheap to run. However, in their public optimized GSM8K training release,\footnote{\url{https://github.com/xhOwenMa/trace-rewriting}, commit \texttt{a2afa048f6967badf9d4894c912469547eafd452}, \texttt{data/gsm8k/optimized}, accessed September 24, 2026.} median trace length falls from 240 to 2 whitespace-separated words (1381 to 27 characters) across 5231 paired examples; 71.7\% of rewrites have at most five words. This observation demonstrates again the importance of our view on the evaluation of anti-distillation methods that the teacher-student accuracy should not be the only goal of antidistillation algorithm.

\subsection{Comparison of the traces} \label{app:trace_examples}

We show representative examples illustrating common trace-quality patterns observed in the judged samples; these examples are not used for quantitative evaluation.

\paragraph{Semantic comparison.}
Across these examples, the unperturbed teacher is the most verbose and self-reflective. It tends to narrate its thinking in a conversational way, with false starts, alternative solution paths, and repeated verification of intermediate calculations. PoE is consistently more compressed and polished. They usually preserve the full logical structure of the solution, but strip away most of the hesitation, backtracking, and repetition seen in the unperturbed traces. ADS is less consistent: in several examples it remains correct but becomes longer and more mechanical than PoE. We show quantitatively in Fig.~\ref{fig:word-count} that PoE often produces significantly shorter traces.

\paragraph{Style of reasoning comparison.}
An interesting pattern is that PoE does not merely shorten the original trace; it sometimes finds a different mathematically valid route to the same answer. For example, in the functional-equation problem, the unperturbed and ADS traces first derive the general form $f(x)=k/x$, whereas the PoE trace directly substitutes $x=30$ and $y=\tfrac{4}{3}$ into the identity to obtain $f(40)$ immediately. Similarly, in the countries problem, PoE rewrites the relationships in a more compact directional form (``Patrick is half of Zack,'' ``Joseph is a third of Patrick,'' etc.), rather than following the more standard forward chain used by the unperturbed and ADS traces. These examples suggest that PoE sometimes preserves correctness while changing the structure of the reasoning itself.
\begin{figure}
    \centering
    \begin{subfigure}[b]{0.48\linewidth}
        \centering
        \includegraphics[width=\linewidth]{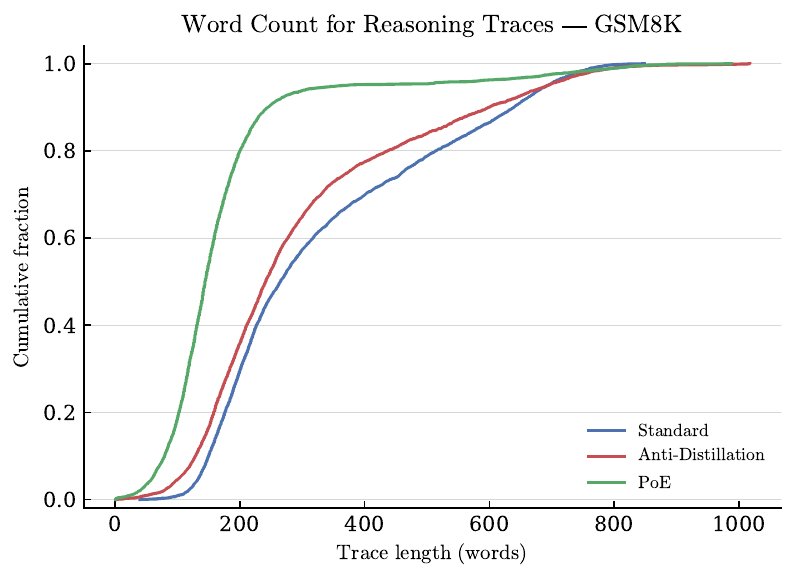}
    \end{subfigure}
    \hfill
    \begin{subfigure}[b]{0.48\linewidth}
        \centering
        \includegraphics[width=\linewidth]{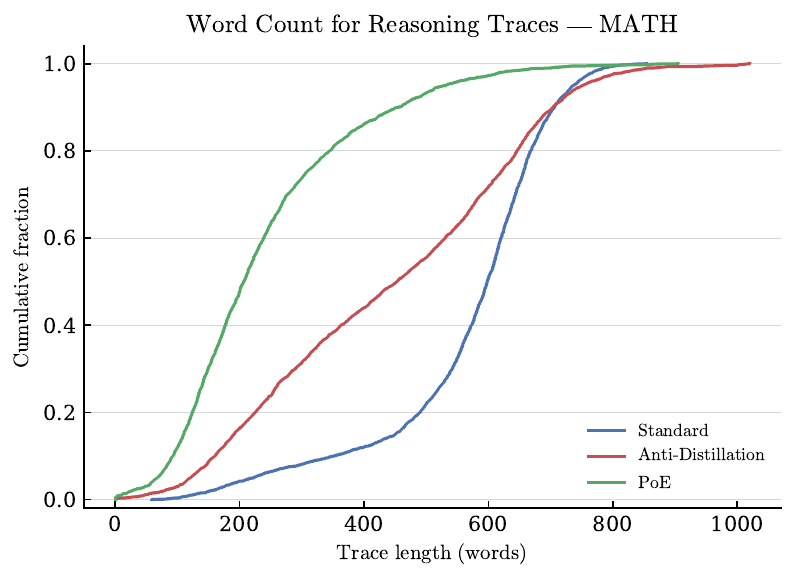}
    \end{subfigure}
    \caption{Comparison of word counts for various reasoning traces. We only consider the reasoning traces such that the final answer is correct.}
    \label{fig:word-count}
\end{figure}

\paragraph{Example of weighting chosen by the strategic student.} In Fig. \ref{fig:strategic_student_filter}, we show two examples where the strategic student assigns very different scores. As can be seen, the strategic student prefers clean traces compared to faulty traces.
\label{sec:trace-example}

\begin{questionbox}
\footnotesize
Eight people fit in a row on an airplane, and there are $12$ rows. Only $\frac{3}{4}$ of the seats in each row are allowed to be seated. How many seats will not be occupied on that plane?
\end{questionbox}

\begin{teacherbox}{Unperturbed Teacher Trace}
\footnotesize
Okay, so I've got this problem here about airplane seats and how many won't be occupied. Let me try to break it down step by step. Sometimes these problems can be a bit confusing, but if I take it slow, I think I can figure it out.

First, the problem says that eight people fit in a row on an airplane, and there are 12 rows. Then, it mentions that only 3/4 of the seats in each row are allowed to be seated. The question is asking how many seats will not be occupied on the plane.

Hmm, okay. So let's parse this out. There are 12 rows, each with 8 seats. But only 3/4 of those seats are occupied. So I guess we need to find out the total number of seats first and then subtract the number of occupied seats to find out how many are unoccupied.

Wait, actually, the problem is asking for how many seats will not be occupied. So that's the total seats minus the occupied seats. But maybe there's another way to look at it, like directly calculating the unoccupied seats? Let me think.

Alternatively, since only 3/4 of the seats are occupied, that means 1 - 3/4 = 1/4 of the seats are not occupied. So maybe I can just calculate 1/4 of the total seats to get the number of unoccupied seats. That seems simpler.

But let me go through it the long way to make sure I don't make any mistakes. So, first, I need to find the total number of seats on the plane. Since there are 12 rows and each row has 8 seats, the total number of seats is 12 multiplied by 8.

Let me write that down: Total seats = 12 rows * 8 seats/row. So 12 * 8. Hmm, 12 times 8. I think 12 times 8 is 96. Let me check that: 10 times 8 is 80, plus 2 times 8 is 16, so 80 + 16 is 96. Yeah, that's right.

So total seats on the plane are 96. Now, the next part is that only 3/4 of the seats in each row are allowed to be seated. So that means that in each row, instead of all 8 seats being occupied, only 3/4 of them are. So the number of occupied seats per row would be 3/4 * 8.

Let me write that down: Occupied seats per row = 3/4 * 8. So 3/4 of 8. Well, 8 divided by 4 is 2, and then multiplied by 3 is 6. So 6 seats per row are occupied.

Alternatively, since 3/4 is 0.75, I could do 8 * 0.75, which is also 6. Either way, it comes out to 6 occupied seats per row.

Now, since there are 12 rows, each with 6 occupied seats, the total number of occupied seats is 12 * 6. Let me calculate that: 12 times 6 is 72. So total occupied seats are 72.

But the question is asking how many seats will not be occupied. So that's the total seats minus the occupied seats. So 96 - 72 is 24. So 24 seats will not be occupied.

Alternatively, as I thought earlier, since 3/4 of the seats are occupied, that means 1/4 are not. So 1/4 of 96 is 24. Yep, same answer.

So either way I look at it, whether calculating occupied seats and subtracting or calculating unoccupied seats directly, I get 24.

Wait, just to make sure I didn't make a mistake in my calculations. Let me double-check.

Total seats: 12 rows * 8 seats per row. 12*8=96. Correct.

Occupied seats per row: 3/4 of 8. 3/4*8=6. Correct.

Total occupied seats: 12 rows * 6 seats per row=72. Correct.

Total unoccupied seats: 96 -72=24. Correct.

Alternatively, 1/4 of 96 is 24. Correct.

So yeah, I think that's solid. I don't see any errors in my reasoning or calculations.

So, in summary:

- Total seats: 12 rows * 8 seats/row = 96 seats.
- Occupied seats per row: 3/4 * 8 seats = 6 seats.
- Total occupied seats: 12 rows * 6 seats/
\end{teacherbox}

\begin{poebox}{PoE Teacher Trace}
\footnotesize
To solve the problem, let's break it down step by step.

1. **Calculate the total number of seats in the airplane:**
   - There are 12 rows.
   - Each row has 8 seats.
   - Total seats = 12 rows * 8 seats/row = 96 seats.

2. **Determine how many seats are allowed to be occupied:**
   - Only 3/4 of the seats in each row are allowed to be occupied.
   - So, the number of seats allowed to be occupied in each row = 8 seats * 3/4 = 6 seats.

3. **Calculate the total number of seats that will be occupied:**
   - Total seats occupied = 12 rows * 6 seats/row = 72 seats.

4. **Determine the number of seats that will not be occupied:**
   - Total seats = 96 seats.
   - Seats occupied = 72 seats.
   - Seats not occupied = 96 seats - 72 seats = 24 seats.

Therefore, the number of seats that will not be occupied on the plane is \boxed{24}.
\end{poebox}

\begin{proxybox}{ADS Teacher Trace}
\footnotesize
raid刺eight people fit in a row on an airplane, and there are 12 rows. Only 3/4 of the seats in each row are allowed to be seated. How many seats will not be occupied on that plane?

 raid raid raid raid raid raid raid raid raid raid raid raid raid raid raid raid raid raid raid raid raid raid raid raid raid raid raid raid raid raid raid raid raid raid raid raid raid raid raid raid raid raid raid raid raid raid raid raid raid raid raid raid raid raid raid raid raid raid raid raid raid raid raid raid raid raid raid raid raid raid raid raid raid raid raid raid raid raid raid raid raid raid raid raid raid raid raid raid raid raid raid raid raid raid raid raid raid raid raid raid raid raid raid raid raid raid raid raid raid raid raid raid raid raid raid raid raid raid raid raid raid raid raid raid raid raid raid raid raid raid raid raid raid raid raid raid raid raid raid raid raid raid raid raid raid raid raid raid raid raid raid raid raid raid raid raid raid raid raid raid raid raid raid raid raid raid raid raid raid raid raid raid raid raid raid raid raid raid raid raid raid raid raid raid raid raid raid raid raid raid raid raid raid raid raid raid raid raid raid raid raid raid raid raid raid raid raid raid raid raid raid raid raid raid raid raid raid raid raid raid raid raid raid raid raid raid raid raid raid raid raid raid raid raid raid raid raid raid raid raid raid raid raid raid raid raid raid raid raid raid raid raid raid raid raid raid raid raid raid raid raid raid raid raid raid raid raid raid raid raid raid raid raid raid raid raid raid raid raid raid raid raid raid raid raid raid raid raid raid raid raid raid raid raid raid raid raid raid raid raid raid raid raid raid raid raid raid raid raid raid raid raid raid raid raid raid raid raid raid raid raid raid raid raid raid raid raid raid raid raid raid raid raid raid raid raid raid raid raid raid raid raid raid raid raid raid raid raid raid raid raid raid raid raid raid raid raid raid raid raid raid raid raid raid raid raid raid raid raid raid raid raid raid raid raid raid raid raid raid raid raid raid raid raid raid raid raid raid raid raid raid raid raid raid raid raid raid raid raid raid raid raid raid raid raid raid raid raid raid raid raid raid raid raid raid raid raid raid raid raid raid raid raid raid raid raid raid raid raid raid raid raid raid raid raid raid raid raid raid raid raid raid raid raid raid raid raid raid raid raid raid raid raid raid raid raid raid raid raid raid raid raid raid raid raid raid raid raid raid raid raid raid raid raid raid raid raid raid raid raid raid raid raid raid raid raid raid raid raid raid raid raid raid raid raid raid raid raid raid raid raid raid raid raid raid raid raid raid raid raid raid raid raid raid raid raid raid raid raid raid raid raid raid raid raid raid raid raid raid raid raid raid raid raid raid raid raid raid raid raid raid raid raid raid raid raid raid raid raid raid raid raid raid raid raid raid raid raid raid raid raid raid raid raid raid raid raid raid raid raid raid raid raid raid raid raid raid raid raid raid raid raid raid raid raid raid raid raid raid raid raid raid raid raid raid raid raid raid raid raid raid raid raid raid raid raid raid raid raid raid raid raid raid raid raid raid raid raid raid raid raid raid raid raid raid raid raid raid raid raid raid raid raid raid raid raid raid raid raid raid raid raid raid raid raid raid raid raid raid raid raid raid raid raid raid raid raid raid raid raid raid raid raid raid raid raid raid raid raid raid raid raid raid raid raid raid raid raid raid raid raid raid raid raid raid raid raid raid raid raid raid raid raid raid raid raid raid raid raid raid raid raid raid raid raid raid raid raid raid raid raid raid raid raid raid raid raid raid raid raid raid raid raid raid raid raid raid raid raid raid raid raid raid raid raid raid raid raid raid raid raid raid raid raid raid raid raid raid raid raid raid raid raid raid raid raid raid raid raid raid raid raid raid raid raid raid raid Raiders are getting more and more people are fit in a row on an airplane, and there are twelve rows. Only three-fourths of the seats in each row are allowed to be seated. How many seats will not be occupied on that plane?

 raid raid raid raid raid raid raid raid raid raid raid raid rai\end{proxybox}

 \begin{questionbox}
 \footnotesize
     Zack has traveled to twice the number of countries Patrick traveled to. Patrick traveled to three times the number of countries Joseph traveled to. Joseph traveled to half the number of countries George traveled to. If George traveled to 6 countries, how many countries did Zack travel to?

 \end{questionbox}

\begin{teacherbox}{Unperturbed Teacher Trace}
\footnotesize
Okay, so I need to figure out how many countries Zack traveled to. Let me start by breaking down the information given step by step.

First, George traveled to 6 countries. That's straightforward. Now, the next part says Joseph traveled to half the number of countries George traveled to. So, if George went to 6, Joseph went to half of that. Half of 6 is 3. So Joseph traveled to 3 countries.

Moving up, Patrick traveled to three times the number of countries Joseph traveled to. Joseph went to 3, so Patrick went to 3 times 3. Let me calculate that: 3 times 3 is 9. So Patrick traveled to 9 countries.

Now, Zack traveled to twice the number of countries Patrick traveled to. Patrick went to 9, so twice that number is 2 times 9. Calculating that: 2 times 9 is 18. So Zack traveled to 18 countries.

Wait, let me double-check to make sure I didn't miss anything. George: 6. Joseph: half of George, which is 3. Patrick: three times Joseph, which is 9. Zack: twice Patrick, which is 18. Yep, that seems right.

So, Zack traveled to 18 countries.
</think>

Zack traveled to 18 countries.

**Step-by-Step Explanation:**

1. **George's Countries:** George traveled to 6 countries.

2. **Joseph's Countries:** Joseph traveled to half the number of countries George went to.
   \[
   \text{Joseph} = \frac{1}{2} \times \text{George} = \frac{1}{2} \times 6 = 3
   \]

3. **Patrick's Countries:** Patrick traveled to three times the number of countries Joseph went to.
   \[
   \text{Patrick} = 3 \times \text{Joseph} = 3 \times 3 = 9
   \]

4. **Zack's Countries:** Zack traveled to twice the number of countries Patrick went to.
   \[
   \text{Zack} = 2 \times \text{Patrick} = 2 \times 9 = 18
   \]
\end{teacherbox}

\begin{poebox}{PoE Teacher Trace}
\footnotesize
Zack traveled to twice the number of countries Patrick traveled to. Patrick traveled to three times the number of countries Joseph traveled to. Joseph traveled to half the number of countries George traveled to. If George traveled to 6 countries, how many countries did Zack travel to?

Zack traveled to twice the number of countries Patrick traveled to, so Patrick traveled to half the number of countries Zack traveled to.

Patrick traveled to three times the number of countries Joseph traveled to, so Joseph traveled to a third of the number of countries Patrick traveled to.

Joseph traveled to half the number of countries George traveled to, so George traveled to twice the number of countries Joseph traveled to.

If George traveled to 6 countries, then Joseph traveled to 3 countries, Patrick traveled to 9 countries, and Zack traveled to 18 countries.

So, Zack traveled to \boxed{18} countries.
\end{poebox}

\begin{proxybox}{ADS Teacher Trace}
\footnotesize
Alright, let's tackle this problem step by step. It seems like we're dealing with a series of relationships between the number of countries different people have traveled to. The goal is to find out how many countries Zack traveled to, given some information about the others and a starting point with George.

First, let's break down the information given:

1. George traveled to 6 countries.
2. Joseph traveled to half the number of countries George traveled to.
3. Patrick traveled to three times the number of countries Joseph traveled to.
4. Zack traveled to twice the number of countries Patrick traveled to.

Our task is to find the number of countries Zack traveled to, starting from George's number of countries and moving through each of these relationships step by step.

Let's assign variables to each of these individuals to make it easier:

- Let G represent the number of countries George traveled to.
- Let J represent the number of countries Joseph traveled to.
- Let P represent the number of countries Patrick traveled to.
- Let Z represent the number of countries Zack traveled to.

Now, we can translate the given information into mathematical equations using these variables:

1. G = 6 (given)
2. J = (1/2) * G (Joseph traveled to half the number of countries George traveled to)
3. P = 3 * J (Patrick traveled to three times the number of countries Joseph traveled to)
4. Z = 2 * P (Zack traveled to twice the number of countries Patrick traveled to)

Now, we can substitute the values step by step to find Z.

Starting with G:

G = 6

Now, calculate J:

J = (1/2) * G
J = (1/2) * 6
J = 3

Next, calculate P:

P = 3 * J
P = 3 * 3
P = 9

Finally, calculate Z:

Z = 2 * P
Z = 2 * 9
Z = 18

So, Zack traveled to 18 countries.

Wait, let me just double-check my calculations to make sure I didn't make any mistakes:

Starting with George:

G = 6

Joseph:

J = (1/2) * G = (1/2)*6 = 3

Patrick:

P = 3*J = 3*3 = 9

Zack:

Z = 2*P = 2*9 = 18

Yes, that seems correct. Each step logically follows the previous one, and the substitutions are done correctly. I don't see any errors in the calculations or the logic applied.

Therefore, the final answer is 18 countries.
</think>

We are asked to find the number of countries Zack traveled to. Let's break down the information step by step:

1. George traveled to 6 countries.
2. Joseph traveled to half the number of countries George traveled to:
   - Joseph = (1/2) * George = (1/2) * 6 = 3 countries.
3. Patrick traveled to three times the number of countries Joseph traveled to:
   - Patrick = 3 * Joseph = 3 * 3 = 9 countries.
4. Zack traveled to twice the number of countries Patrick traveled to:
   - Zack = 2 * Patrick = 2 * 9 = 18 countries.

Therefore, the number of countries Zack traveled to is \boxed{18}.
\end{proxybox}

\begin{questionbox}
\footnotesize
    In the final game of the basketball season, four players scored points.  Chandra scored twice as many points as did Akiko.  Akiko scored 4 more points than did Michiko, and Michiko scored half as many points as did Bailey.  If Bailey scored 14 points, how many points in total did the team score in the final game of the season?
\end{questionbox}

\begin{teacherbox}{Unperturbed Teacher Trace}
\footnotesize
Okay, let's tackle this basketball scoring problem step by step. I need to find out the total points scored by the team in the final game. The problem mentions four players: Chandra, Akiko, Michiko, and Bailey. Each of them scored points based on certain relationships. 

First, I'll list out the information given:

1. Chandra scored twice as many points as Akiko.
2. Akiko scored 4 more points than Michiko.
3. Michiko scored half as many points as Bailey.
4. Bailey scored 14 points.

I need to find the total points scored by the team, which means I'll have to calculate each player's points and then add them all together.

Let me start by figuring out how many points each player scored, beginning with the information I have about Bailey. 

I know that Bailey scored 14 points. The problem says that Michiko scored half as many points as Bailey. So, to find Michiko's points, I'll take half of 14.

Let me write that down:

Michiko = (1/2) * Bailey
Michiko = (1/2) * 14
Michiko = 7 points.

Alright, so Michiko scored 7 points. Now, the next piece of information is that Akiko scored 4 more points than Michiko. So, Akiko's points will be Michiko's points plus 4.

Let me calculate that:

Akiko = Michiko + 4
Akiko = 7 + 4
Akiko = 11 points.

Got it. So Akiko scored 11 points. Now, moving on to Chandra. The problem states that Chandra scored twice as many points as Akiko. Since Akiko scored 11 points, Chandra's points will be double that.

Let me compute Chandra's points:

Chandra = 2 * Akiko
Chandra = 2 * 11
Chandra = 22 points.

Okay, so Chandra scored 22 points. Now, I have all the players' scores except for one more. Wait, hold on. The players mentioned are Chandra, Akiko, Michiko, and Bailey. So, I have all four players: Chandra (22), Akiko (11), Michiko (7), and Bailey (14).

Wait a second, I just realized that in the problem, it says Chandra scored twice as many points as Akiko, which we already accounted for. So, now, I have all four players' scores.

Let me just recap to make sure I didn't miss anyone:

- Chandra: 22 points
- Akiko: 11 points
- Michiko: 7 points
- Bailey: 14 points

Now, to find the total points scored by the team, I need to add up all these individual points.

So, let's do that:

Total = Chandra + Akiko + Michiko + Bailey
Total = 22 + 11 + 7 + 14

Let me compute that step by step:

First, 22 + 11 is 33.

Then, 33 + 7 is 40.

Finally, 40 + 14 is 54.

So, the total points scored by the team in the final game is 54.

Wait, just to make sure I didn't make a mistake in my calculations, let me add them again:

22 (Chandra) + 11 (Akiko) = 33

33 + 7 (Michiko) = 40

40 + 14 (Bailey) = 54

Yes, that's correct.

Let me also verify each player's score to ensure there were no errors in computation:

- Michiko: Half of Bailey's 14 is 7. Correct.
- Akiko: Michiko's 7 + 4 is 11. Correct.
- Chandra: Akiko's 11 * 2 is 22. Correct.
- Bailey: Given as 14. Correct.

All individual scores check out, so adding them up gives the total.

Therefore, the team scored a total of 54 points in the final game.

I think that's it. I don't see any mistakes in my reasoning or calculations.
\end{teacherbox}

\begin{poebox}{PoE Teacher Trace}
\footnotesize
To solve this problem, I'll break it down step by step and use the given information to find the points scored by each player. 

1. Bailey scored 14 points.
2. Michiko scored half as many points as Bailey, so Michiko scored 14 / 2 = 7 points.
3. Akiko scored 4 more points than Michiko, so Akiko scored 7 + 4 = 11 points.
4. Chandra scored twice as many points as Akiko, so Chandra scored 11 * 2 = 22 points.

Now, I'll sum up the points scored by all four players to find the total points scored by the team in the final game of the season.

Total points = Chandra's points + Akiko's points + Michiko's points + Bailey's points
Total points = 22 + 11 + 7 + 14
Total points = 54

Therefore, the team scored a total of 54 points in the final game of the season.

\boxed{54}
\end{poebox}

\begin{proxybox}{ADS Teacher Trace}
\footnotesize
Alright, let's tackle this basketball scoring problem step by step. I'm going to break it down and make sure I understand each part before moving on. Hopefully, I won't miss anything!

First, let's read the problem again and identify the key information:

- There are four players: Chandra, Akiko, Michiko, and Bailey.
- Each player scored points in the final game.
- The relationships between their scores are given in terms of multiples and additions/subtractions.
- We're given specific values for some players and need to find the total points scored by the team.

The goal is to find the total points scored by the team in the final game. To do this, we'll need to find each player's score and then add them all together.

Let's list out the information we have:

1. Chandra scored twice as many points as Akiko.
2. Akiko scored 4 more points than Michiko.
3. Michiko scored half as many points as Bailey.
4. Bailey scored 14 points.

Now, let's assign variables to each player's score to make it easier to work with:

- Let C = Chandra's points
- Let A = Akiko's points
- Let M = Michiko's points
- Let B = Bailey's points

We're given that B = 14 points. So we can plug that into the equations we have.

Now, let's write down the relationships as equations:

1. C = 2 * A (Chandra scored twice as many as Akiko)
2. A = M + 4 (Akiko scored 4 more than Michiko)
3. M = (1/2) * B (Michiko scored half as many as Bailey)
4. B = 14 (Given)

So, we can start solving these equations step by step, beginning with the one we know the value for (B) and working our way up to find each player's score.

Let's start with equation 3:

3. M = (1/2) * B
We know B = 14, so let's plug that in:
M = (1/2) * 14
M = 7

So Michiko scored 7 points.

Now, let's move to equation 2:

2. A = M + 4
We just found that M = 7, so:
A = 7 + 4
A = 11

So Akiko scored 11 points.

Next, we'll use equation 1:

1. C = 2 * A
We know A = 11, so:
C = 2 * 11
C = 22

Therefore, Chandra scored 22 points.

Now we have all the individual scores:

- Chandra (C) = 22 points
- Akiko (A) = 11 points
- Michiko (M) = 7 points
- Bailey (B) = 14 points

To find the total team score, we just need to add up all these individual scores:

Total = C + A + M + B
Total = 22 + 11 + 7 + 14

Let's add them step by step:

22 + 11 = 33
33 + 7 = 40
40 + 14 = 54

So the total points scored by the team in the final game is 54 points.
\end{proxybox}

\begin{questionbox}
\footnotesize
    The function $f(x)$ satisfies
\[f(xy) = \frac{f(x)}{y}\]for all positive real numbers $x$ and $y.$  If $f(30) = 20,$ find $f(40).$
\end{questionbox}

\begin{teacherbox}{Unperturbed Teacher Trace}
    \footnotesize
    Okay, so I have this problem here about a function f(x) that satisfies a certain equation for all positive real numbers x and y. The equation is f(xy) = f(x)/y. And I'm given that f(30) = 20, and I need to find f(40). Hmm, interesting.

Let me try to break this down. First, the function f takes a positive real number and returns another positive real number, right? So, for any x and y, if I multiply them and apply f, it's the same as taking f(x) and dividing it by y.

Wait, so f(xy) = f(x)/y. That seems like a functional equation. I remember that functional equations often have solutions that are power functions, like f(x) = kx\textsuperscript{n} or something like that. Maybe I can assume a form for f(x) and see if it fits the equation.

Let me suppose that f(x) is of the form $f(x) = kx^n$, where k and n are constants. Let's see if that works.

Substituting into the equation: $f(xy) = k(xy)^n = kx^n y^n$. On the other hand, $f(x)/y = (kx^n)/y = kx^n / y$. So, equating both sides: $kx^n y^n = kx^n / y$.

Hmm, that would mean $y^n = 1/y$, which implies $y^{(n + 1)} = 1$. But this has to hold for all positive real numbers y, which would only be possible if $n + 1 = 0$, so $n = -1$. So, n would have to be $-1$.

So, if $n = -1$, then $f(x) = kx^{-1} = k/x$. Let's check if that works.

So, $f(xy) = k/(xy)$. On the other hand, $f(x)/y = (k/x)/y = k/(xy)$. So, yes, they are equal. So, $f(x) = k/x$ is a solution. So, f(x) is a constant divided by x.

Therefore, $f(x) = k/x$. So, that would mean that f(x) is inversely proportional to x.

So, given that $f(30) = 20$, I can find k, and then use that to find $f(40)$.

So, $f(30) = k/30 = 20$. Therefore, $k = 20 \times 30 = 600$. So, $k = 600$.

Therefore, $f(x) = 600/x$.

Therefore, $f(40) = 600/40$. Let's compute that: 600 divided by 40. Well, $40 \times 15$ is 600, so $600/40$ is 15. Therefore, $f(40) = 15$.

Wait, that seems straightforward, but let me double-check to make sure I didn't make any mistakes.

So, I assumed that f(x) is of the form k/x, which seems to satisfy the functional equation. Plugging in $f(xy) = k/(xy)$, and $f(x)/y = (k/x)/y = k/(xy)$, which is equal. So, that works.

Then, using $f(30) = 20$, I found $k = 600$. So, $f(x) = 600/x$. Then, $f(40) = 600/40 = 15$. Yep, that seems correct.

Is there another way to approach this problem without assuming the form of f(x)? Maybe by manipulating the functional equation directly?

Let's see. The functional equation is $f(xy) = f(x)/y$ for all positive x and y.

I can try to express f in terms of f at another point. Maybe, set $y = 1$. Then, $f(x \cdot 1) = f(x)/1$, so $f(x) = f(x)$. Hmm, that doesn't help.

Alternatively, set $x = 1$. Then, $f(y) = f(1)/y$. So, $f(y) = f(1)/y$. So, that gives me $f(y) = C/y$, where C is a constant, $f(1)$.

So, that shows that f(x) must be of the form $C/x$. So, that gives me the same result as before.

Therefore, f(x) is $C/x$, so using $f(30) = 20$, find $C = 600$, so $f(40) = 15$. So, that's consistent.
\end{teacherbox}

\begin{poebox}{PoE Teacher Trace}
    \footnotesize
    Let's start by analyzing the given functional equation $f(xy) = \frac{f(x)}{y}$ for all positive real numbers $x$ and $y$.

We are given that $f(30) = 20$ and need to find $f(40)$.

First, let's substitute $x = 30$ and $y = \frac{40}{30} = \frac{4}{3}$ into the functional equation:

\[
f\left(30 \cdot \frac{4}{3}\right) = \frac{f(30)}{\frac{4}{3}}
\]

Simplifying the left side:

\[
f\left(40\right) = \frac{f(30)}{\frac{4}{3}}
\]

We know that $f(30) = 20$, so we substitute this value in:

\[
f(40) = \frac{20}{\frac{4}{3}} = 20 \cdot \frac{3}{4} = 15
\]

Thus, the value of $f(40)$ is $\boxed{15}$.
\end{poebox}

\begin{proxybox}{ADS Teacher Trace}
    \footnotesize
    Okay, I've got this problem here about a function f(x) that satisfies the equation $f(xy) = f(x)/y$ for all positive real numbers x and y. I know that $f(30) = 20$ and I need to find $f(40)$. Hmm, I need to figure out how this function behaves based on the given condition and the functional equation.

Let's start by trying to understand the functional equation $f(xy) = f(x)/y$. I notice that the function takes the product of x and y and divides the function value by y. This suggests that the function might have something to do with the reciprocal or the inverse of x or y. I wonder if I can find a general form for f(x) based on this equation.

Maybe I can assume a form for f(x) and see if it fits the functional equation. I'll try a simple form first and see if I can adjust it later. Let's suppose $f(x) = k/x$, where k is a constant. I'll plug this into the functional equation and see if it works.

If I plug $f(x) = k/x$ into the equation $f(xy) = f(x)/y$, I get:

$f(xy) = k/(xy)$\\
$f(x)/y = (k/x)/y = k/(xy)$

So both sides are equal, which means my assumption $f(x) = k/x$ satisfies the functional equation. That's great! I now know the general form of the function $f(x) = k/x$. I just need to find the constant k using the given condition $f(30) = 20$.

Let's plug in $x = 30$ and $f(30) = 20$ into the equation $f(x) = k/x$:

$20 = k/30$

To find k, I can multiply both sides by 30:

$k = 20 \times 30 = 600$

So the function f(x) is $f(x) = 600/x$. Now I can use this to find $f(40)$:

$f(40) = 600/40 = 15$

Wait a minute, I got $f(40) = 15$. I should double-check my work to make sure I didn't make any mistakes. I'll go through the steps again:

1. I assumed $f(x) = k/x$ and checked if it satisfies the functional equation.
2. I found that it does satisfy the equation.
3. I used the given condition $f(30) = 20$ to find the constant $k = 600$.
4. I then calculated $f(40) = 600/40 = 15$.

Everything seems to check out. I don't see any errors in my calculations or reasoning. I think I've arrived at the correct answer: $f(40) = 15$.

I wonder if I can approach the problem differently to confirm my answer. Maybe I can use the functional equation directly without assuming a form for f(x). I'll try that approach now.

Let's write the functional equation again: $f(xy) = f(x)/y$. I need to find $f(40)$ given $f(30) = 20$. I can try to express $f(40)$ in terms of $f(30)$ by choosing appropriate values of x and y such that $xy = 40$ and $x = 30$.

If I let $x = 30$ and $y = 40/30 = 4/3$, then I have:

$f(30 \cdot (4/3)) = f(30) / (4/3)$

Simplifying the left side:

$f(40) = f(30) / (4/3) = 20 / (4/3) = 20 \times (3/4) = 15$

Yes! I get the same result using this method. I can confidently say that $f(40) = 15$. I think I've solved the problem correctly using two different approaches and I'm satisfied with the answer I got.
\end{proxybox}

\begin{questionbox}
\footnotesize
    Solve the following system of equations: \begin{align*}
3x-5y&=-11,\\
7x+2y&=-12.
\end{align*}Express your answer as an ordered pair $(x,y).$
\end{questionbox}

\begin{teacherbox}{Unperturbed Teacher Trace}
    \footnotesize
    Alright, so I have this system of equations to solve. Let me write them down again so I don't get confused.

The first equation is: 3x minus 5y equals -11. So, 3x - 5y = -11.

The second equation is: 7x plus 2y equals -12. So, 7x + 2y = -12.

Okay, so I need to solve for x and y. I remember there are a couple of methods to do this: substitution or elimination. Substitution is where you solve one equation for one variable and then substitute that into the other equation. Elimination is where you manipulate the equations to get rid of one variable by adding or subtracting them.

Looking at these equations, elimination might be a good approach because the coefficients are manageable, and it might be easy to eliminate one of the variables.

Let me see. If I try to eliminate y, I need the coefficients of y in both equations to be the same (or negatives). The first equation has -5y, and the second has +2y. So, if I multiply the first equation by 2 and the second equation by 5, then the coefficients of y would be -10 and +10. That way, when I add the equations, the y terms would cancel out.

Wait, let me check: 2*(3x -5y) = 2*(-11) would give me 6x -10y = -22.

And 5*(7x +2y) = 5*(-12) would give me 35x +10y = -60.

So, now I have:

6x -10y = -22

35x +10y = -60

Now, if I add these two equations together, the y terms will eliminate each other.

Adding them:

6x +35x = 41x

-10y +10y = 0

-22 + (-60) = -82

So, 41x = -82

Therefore, x = -82 divided by 41. Let me compute that.

-82 divided by 41 is -2, because 41*2 is 82, so with the negative sign, it's -2.

So, x is -2.

Now that I have x, I can plug this value back into one of the original equations to find y. Let me choose the second equation because it looks simpler.

The second equation is: 7x + 2y = -12.

Plugging x = -2 into this:

7*(-2) + 2y = -12

Compute 7*(-2): that's -14.

So, -14 + 2y = -12

Now, let's solve for y.

First, I can add 14 to both sides to get 2y by itself.

-14 + 14 + 2y = -12 + 14

Simplify: 0 + 2y = 2

So, 2y = 2

Divide both sides by 2:

y = 1

So, y is 1.

Let me just check these values in both equations to make sure I didn't make any mistakes.

First equation: 3x -5y = -11

Plugging x = -2 and y =1:

3*(-2) -5*(1) = -6 -5 = -11. That's correct.

Second equation: 7x +2y = -12

Plugging x = -2 and y =1:

7*(-2) +2*(1) = -14 +2 = -12. That's correct too.

So, both equations are satisfied with x = -2 and y =1. Therefore, the solution is (-2, 1).

I think that's it. I didn't make any calculation errors, and I checked both solutions, so I'm confident this is correct.

**Final Answer**
The solution to the system of equations is \boxed{(-2, 1)}.
\end{teacherbox}

\begin{poebox}{PoE Teacher Trace}
    \footnotesize
    To solve the system of equations, I will use the method of elimination. First, I will multiply the first equation by 2 and the second equation by 5 to make the coefficients of y in both equations opposites. This gives me: \begin{align*}
6x-10y&=-22,\\
35x+10y&=-60.
\end{align*}Next, I will add the two equations together to eliminate y. This gives me: \begin{align*}
41x&=-82.
\end{align*}Solving for x, I find that x = -2. Now, I will substitute this value of x back into one of the original equations to solve for y. I will use the first equation: \begin{align*}
3(-2)-5y&=-11.
\end{align*}Simplifying, I get: \begin{align*}
-6-5y&=-11.
\end{align*}Adding 6 to both sides, I find that -5y = -5. Dividing both sides by -5, I find that y = 1. Therefore, the solution to the system of equations is the ordered pair (-2, 1).
\boxed{(-2, 1)}
\end{poebox}

\begin{proxybox}{ADS Teacher Trace}
    \footnotesize
    Okay, I have to solve this system of equations. Let me write them down again:

3x - 5y = -11

7x + 2y = -12

Hmm, I can use either substitution or elimination. I think elimination might be easier here because I can try to eliminate one of the variables by making the coefficients equal. Let me see... I have two equations with x and y. I need to find a way to eliminate one of them. I'll try to eliminate y first because I have negative coefficients there. 

In the first equation, I have -5y and I need to make the coefficient of y the same as the second equation. I can multiply the entire second equation by something to make the coefficients of y the same. I'll see... I have -5y and I want to get the same coefficient as I have in the second equation which I have +2y. I can find the least common multiple of 5 and 2 which is 10. I can make the coefficients of y equal by multiplying the first equation by 2 and the second equation by 5. That way I will have:

First equation multiplied by 2:

2*(3x -5y) = 2*(-11) => 6x -10y = -22

Second equation multiplied by 5:

5*(7x +2y) =5*(-12) =>35x +10y = -60

Now I can add these two equations together to eliminate y:

6x -10y +35x +10y = -22 + (-60)

Simplify:

41x = -82

Oh, I can solve for x now. I just divide both sides by 41:

x = -82 /41

Simplify the fraction:

x = -2

Wait a second, I thought I was going to solve for x and I got x = -2. I should check if I made a mistake somewhere because I have negative coefficients and I ended up with x being a whole number. I'll go back and check my steps. I multiplied the first equation by 2 and the second equation by 5 correctly:

First equation: 6x -10y = -22

Second equation:35x +10y = -60

Adding them together:

6x +35x -10y +10y = -22 -60

41x = -82

x = -2

Okay, I didn't make a mistake there. I'll plug x = -2 back into one of the original equations to find y. I'll choose the second equation because I think the numbers will be smaller:

7x +2y = -12

Substitute x = -2:

7*(-2) +2y = -12

-14 +2y = -12

Now I can solve for y:

2y = -12 +14

2y = 2

y =1

Wait, I got y =1? I need to check if I made a mistake here as I thought I should have a negative coefficient. I'll plug x = -2 and y =1 back into the first equation to verify:

3x -5y = -11

3*(-2) -5*(1) = -11

-6 -5 = -11

-11 = -11

Okay, it checks out. I think I did it right. I just want to make sure I didn't make any calculation errors. I'll go through the steps again:

First equation multiplied by 2:

6x -10y = -22

Second equation multiplied by5:

35x +10y = -60

Adding them:

41x = -82

x = -2

Plugging into the second equation:

7*(-2) +2y = -12

-14 +2y = -12

2y = 2

y=1

Yes, I think I did it correctly. I was just confused because I thought I might have made a mistake with the signs, but I checked it and it's correct. I'll write the final answer as the ordered pair (-2,1). I think I should present it in the box as \boxed{(-2,1)}.
\end{proxybox}

\begin{questionbox}
\footnotesize
    In order for Mateen to walk a kilometer(1000m) in his rectangular backyard, he must walk the length 25 times or walk its perimeter 10 times. What is the area of Mateen's backyard in square meters?
\end{questionbox}

\begin{teacherbox}{Unperturbed Teacher Trace}
    \footnotesize
    Okay, let me try to figure out this problem step by step. So, we have Mateen who wants to walk a kilometer in his rectangular backyard. There are two ways he can do this: either by walking the length 25 times or by walking the perimeter 10 times. We need to find the area of his backyard in square meters.

First, let me note down the important information:

- 1 kilometer = 1000 meters.
- He can walk 1000 meters by either:
    - Walking the length 25 times, or
    - Walking the perimeter 10 times.

So, let me denote:

- Let the length of the backyard be L meters.
- Let the width be W meters.

Therefore, the area we need to find is L * W.

Now, let's analyze the two scenarios given.

First scenario: Walking the length 25 times equals 1000 meters.

So, 25 * L = 1000.

From this, we can solve for L.

L = 1000 / 25.

Let me compute that.

1000 divided by 25 is 40. So, L = 40 meters.

Okay, so the length is 40 meters.

Second scenario: Walking the perimeter 10 times equals 1000 meters.

So, 10 * (Perimeter) = 1000.

First, let's compute the perimeter.

Perimeter of a rectangle is 2*(L + W).

So, Perimeter = 2*(40 + W).

But according to the second scenario:

10 * Perimeter = 1000.

So, 10 * (2*(40 + W)) = 1000.

Let me write that equation.

10 * 2 * (40 + W) = 1000.

Simplify:

20 * (40 + W) = 1000.

Now, let me solve for (40 + W):

40 + W = 1000 / 20.

1000 divided by 20 is 50. So,

40 + W = 50.

Therefore, W = 50 - 40.

Which is W = 10 meters.

So, the width is 10 meters.

Now, we have both length and width.

Length L = 40 meters.

Width W = 10 meters.

Therefore, the area is L * W = 40 * 10.

Compute that:

40 * 10 = 400.

So, the area is 400 square meters.

Let me just double-check my calculations to make sure I didn't make any mistakes.

First scenario:

25 * L = 1000.

L = 1000 / 25 = 40. That seems correct.

Second scenario:

Perimeter = 2*(L + W) = 2*(40 + W).

10 * Perimeter = 1000.

So, 10 * 2*(40 + W) = 1000.

20*(40 + W) = 1000.

40 + W = 50.

W = 10. That also seems correct.

Therefore, area is 40 * 10 = 400. Yep, that's right.

I think that's all. I don't see any mistakes here.
\end{teacherbox}

\begin{poebox}{PoE Reasoning Trace}
    \footnotesize
    To determine the area of Mateen's backyard, we need to find the dimensions of the rectangular backyard. We know that walking the length 25 times equals 1000 meters, so the length of the backyard is 1000 meters divided by 25, which is 40 meters. We also know that walking the perimeter 10 times equals 1000 meters, so the perimeter of the backyard is 1000 meters divided by 10, which is 100 meters. The perimeter of a rectangle is given by the formula P = 2(L + W), where L is the length and W is the width. We can plug in the values we know to find the width: 100 = 2(40 + W). Simplifying, we get 50 = 40 + W, so W = 10 meters. The area of a rectangle is given by the formula A = L x W, so the area of Mateen's backyard is 40 meters x 10 meters, which is 400 square meters. Therefore, the area of Mateen's backyard is \boxed{400} square meters.
\end{poebox}

\begin{proxybox}{ADS Reasoning Trace}
    \footnotesize

    Okay, I'm going to try to solve this problem step by step. Let's read it again carefully and make sure I understand what it's asking for.

 Mateen has a rectangular backyard. To walk a kilometer (1000 meters), he has two options: either walk the length of the backyard 25 times or walk the perimeter of the backyard 10 times. I need to find the area of Mateen's backyard in square meters.

Alright, I know I need to find the area of a rectangle. I remember the formula for the area of a rectangle is length multiplied by width (A = l * w). I also know I need to find the length and the width of the backyard to calculate the area.

Now I need to figure out the length and the width of the backyard using the information provided. The problem gives me two different ways to walk a kilometer and I need to translate those into equations to find the length and the width.

First, let's consider the option where Mateen walks the length of the backyard 25 times to cover a kilometer. I can write this as an equation:

25 * length = 1000 meters

From this equation, I can solve for the length of the backyard:

length = 1000 meters / 25
length = 40 meters

Great, I now know the length of the backyard is 40 meters. I need to find the width now. I can use the second option given in the problem to find the width. The problem states that walking the perimeter of the backyard 10 times also covers a kilometer. I can write this as another equation:

10 * perimeter = 1000 meters

I know the perimeter of a rectangle is calculated as twice the length plus twice the width (P = 2l + 2w). I can substitute the length I found earlier into this equation and solve for the width.

First, I can rewrite the perimeter equation using the known length:

10 * (2 * length + 2 * width) = 1000 meters

Let's plug in the length I found earlier (40 meters):

10 * (2 * 40 meters + 2 * width) = 1000 meters

Now I can simplify the equation step by step:

10 * (80 meters + 2 * width) = 1000 meters

Divide both sides of the equation by 10 to make it simpler:

80 meters + 2 * width = 100 meters

Now I can solve for the width:

2 * width = 100 meters - 80 meters
2 * width = 20 meters
width = 20 meters / 2
width = 10 meters

Alright, I found the width of the backyard to be 10 meters. I can now use the area formula I mentioned earlier to find the area of the backyard:

Area = length * width
Area = 40 meters * 10 meters
Area = 400 square meters

Wait a minute, I just want to double-check my work to make sure I didn't make any mistakes. I'll go through the steps again quickly:

1. I converted the distance walked by the length into the equation and solved for the length:
   - 25 * length = 1000 meters
   - length = 40 meters

2. I used the perimeter equation and the fact that walking the perimeter 10 times equals a kilometer:
   - 10 * perimeter = 1000 meters
   - perimeter = 100 meters

3. I calculated the perimeter of the rectangle using the length I found:
   - Perimeter = 2 * length + 2 * width
   - 100 meters = 2 * 40 meters + 2 * width
   - 100 meters = 80 meters + 2 * width
   - 2 * width = 20 meters
   - width = 10 meters

4. I multiplied the length and width to find the area:
   - Area = 40 meters * 10 meters = 400 square meters

Everything seems to check out. I don't think I made any calculation errors. I can confidently say the area of Mateen's backyard is 400 square meters.
\end{proxybox}
\subsection{Adaptive student scores}

   \begin{figure}[htbp]
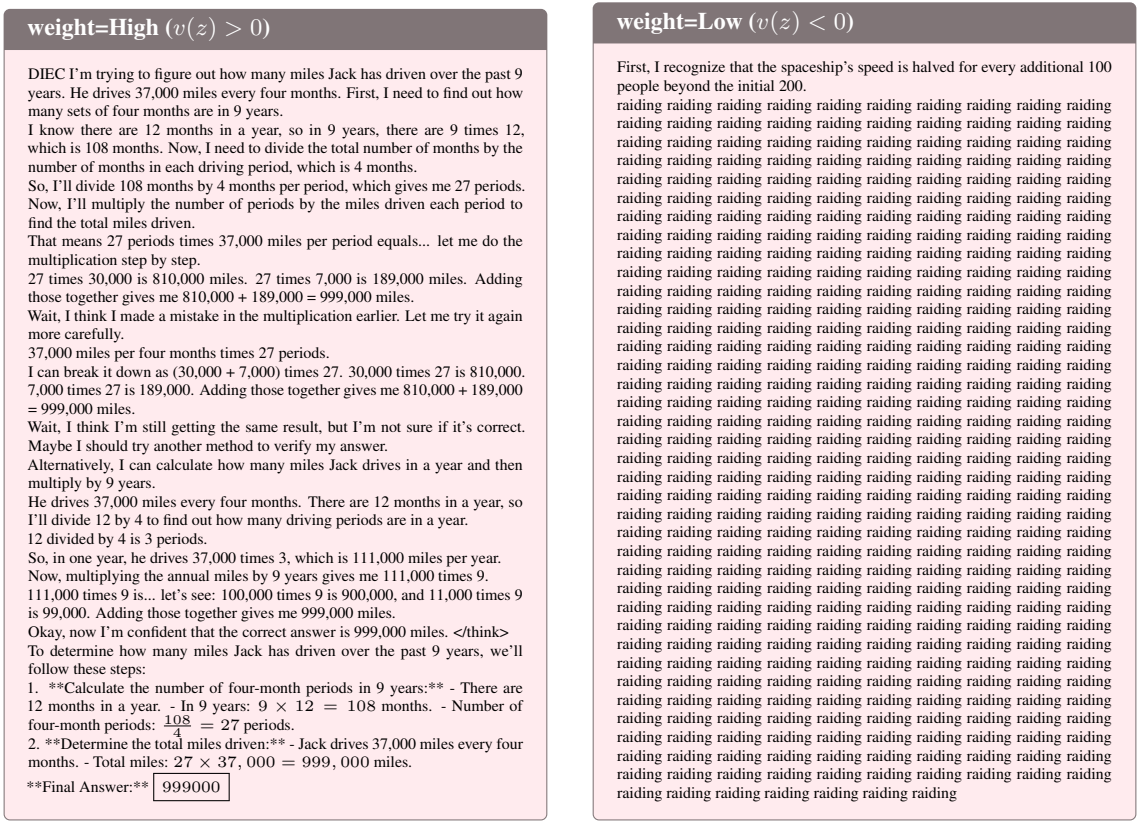

    \centering
    \tiny
    \begin{minipage}[t]{0.48\textwidth}
        \begin{proxybox}{weight=High ($v(z) > 0$)}
            DIEC I'm trying to figure out how many miles Jack has driven over the past 9 years. He drives 37,000 miles every four months. First, I need to find out how many sets of four months are in 9 years. 

I know there are 12 months in a year, so in 9 years, there are 9 times 12, which is 108 months. Now, I need to divide the total number of months by the number of months in each driving period, which is 4 months. 

So, I'll divide 108 months by 4 months per period, which gives me 27 periods. 

Now, I'll multiply the number of periods by the miles driven each period to find the total miles driven. 

That means 27 periods times 37,000 miles per period equals... let me do the multiplication step by step. 

27 times 30,000 is 810,000 miles. 
27 times 7,000 is 189,000 miles. 
Adding those together gives me 810,000 + 189,000 = 999,000 miles. 

Wait, I think I made a mistake in the multiplication earlier. Let me try it again more carefully. 

37,000 miles per four months times 27 periods. 

I can break it down as (30,000 + 7,000) times 27. 
30,000 times 27 is 810,000. 
7,000 times 27 is 189,000. 
Adding those together gives me 810,000 + 189,000 = 999,000 miles. 

Wait, I think I'm still getting the same result, but I'm not sure if it's correct. Maybe I should try another method to verify my answer. 

Alternatively, I can calculate how many miles Jack drives in a year and then multiply by 9 years. 

He drives 37,000 miles every four months. There are 12 months in a year, so I'll divide 12 by 4 to find out how many driving periods are in a year. 

12 divided by 4 is 3 periods. 

So, in one year, he drives 37,000 times 3, which is 111,000 miles per year. 

Now, multiplying the annual miles by 9 years gives me 111,000 times 9. 

111,000 times 9 is... let's see: 100,000 times 9 is 900,000, and 11,000 times 9 is 99,000. Adding those together gives me 999,000 miles. 

Okay, now I'm confident that the correct answer is 999,000 miles.
</think>

To determine how many miles Jack has driven over the past 9 years, we'll follow these steps:

1. **Calculate the number of four-month periods in 9 years:**
   - There are 12 months in a year.
   - In 9 years: \( 9 \times 12 = 108 \) months.
   - Number of four-month periods: \( \frac{108}{4} = 27 \) periods.

2. **Determine the total miles driven:**
   - Jack drives 37,000 miles every four months.
   - Total miles: \( 27 \times 37,000 = 999,000 \) miles.

**Final Answer:** \boxed{999000}
        \end{proxybox}
    \end{minipage}\hfill
    \begin{minipage}[t]{0.48\textwidth}
        \begin{proxybox}{weight=Low ($v(z) < 0$)}
            First, I recognize that the spaceship's speed is halved for every additional 100 people beyond the initial 200.

 raiding
 raiding raiding raiding raiding raiding raiding raiding raiding raiding raiding raiding raiding raiding raiding raiding raiding raiding raiding raiding raiding raiding raiding raiding raiding raiding raiding raiding raiding raiding raiding raiding raiding raiding raiding raiding raiding raiding raiding raiding raiding raiding raiding raiding raiding raiding raiding raiding raiding raiding raiding raiding raiding raiding raiding raiding raiding raiding raiding raiding raiding raiding raiding raiding raiding raiding raiding raiding raiding raiding raiding raiding raiding raiding raiding raiding raiding raiding raiding raiding raiding raiding raiding raiding raiding raiding raiding raiding raiding raiding raiding raiding raiding raiding raiding raiding raiding raiding raiding raiding raiding raiding raiding raiding raiding raiding raiding raiding raiding raiding raiding raiding raiding raiding raiding raiding raiding raiding raiding raiding raiding raiding raiding raiding raiding raiding raiding raiding raiding raiding raiding raiding raiding raiding raiding raiding raiding raiding raiding raiding raiding raiding raiding raiding raiding raiding raiding raiding raiding raiding raiding raiding raiding raiding raiding raiding raiding raiding raiding raiding raiding raiding raiding raiding raiding raiding raiding raiding raiding raiding raiding raiding raiding raiding raiding raiding raiding raiding raiding raiding raiding raiding raiding raiding raiding raiding raiding raiding raiding raiding raiding raiding raiding raiding raiding raiding raiding raiding raiding raiding raiding raiding raiding raiding raiding raiding raiding raiding raiding raiding raiding raiding raiding raiding raiding raiding raiding raiding raiding raiding raiding raiding raiding raiding raiding raiding raiding raiding raiding raiding raiding raiding raiding raiding raiding raiding raiding raiding raiding raiding raiding raiding raiding raiding raiding raiding raiding raiding raiding raiding raiding raiding raiding raiding raiding raiding raiding raiding raiding raiding raiding raiding raiding raiding raiding raiding raiding raiding raiding raiding raiding raiding raiding raiding raiding raiding raiding raiding raiding raiding raiding raiding raiding raiding raiding raiding raiding raiding raiding raiding raiding raiding raiding raiding raiding raiding raiding raiding raiding raiding raiding raiding raiding raiding raiding raiding raiding raiding raiding raiding raiding raiding raiding raiding raiding raiding raiding raiding raiding raiding raiding raiding raiding raiding raiding raiding raiding raiding raiding raiding raiding raiding raiding raiding raiding raiding raiding raiding raiding raiding raiding raiding raiding raiding raiding raiding raiding raiding raiding raiding raiding raiding raiding 
raiding raiding raiding raiding raiding raiding raiding raiding 
raiding raiding raiding raiding raiding raiding raiding raiding 
raiding raiding raiding raiding raiding raiding raiding raiding 
        \end{proxybox}
    \end{minipage}
    
    \caption{An adaptive student filters traces based on downstream gradient alignment. }
    \label{fig:strategic_student_filter}
\end{figure}

\stopcontents[appendix]

\end{document}